\documentclass{article}

\usepackage{PRIMEarxiv}

\usepackage[utf8]{inputenc} 
\usepackage[T1]{fontenc}    
\usepackage{hyperref}       
\usepackage{url}            
\usepackage{booktabs}       
\usepackage{amsfonts}       
\usepackage{nicefrac}       
\usepackage{microtype}      
\usepackage{lipsum}
\usepackage{fancyhdr}       
\usepackage{graphicx}       
\graphicspath{{media/}}     

\usepackage{amsmath,amsfonts,amsthm}
\usepackage{graphicx}
\usepackage{textcomp}
\usepackage{xcolor}
\newcommand{\partitle}[1]{\smallskip \noindent \textbf{#1.}}

\usepackage[mathscr]{eucal}
\usepackage{algorithm, algorithmic}
\usepackage{tikz}
\usepackage{stackengine,enumitem}
\usepackage{booktabs}
\usepackage{subcaption}
\usepackage{mathtools}
\usepackage{multirow}
\usepackage{subfiles}
\usepackage{caption,tabularx,booktabs}
\usepackage{url}
\usepackage{colortbl}
\usepackage{booktabs}

\newtheorem{theorem}{Theorem}
\newtheorem{definition}{Definition}

\makeatletter
\newsavebox\myboxA
\newsavebox\myboxB
\newlength\mylenA

\newcommand*\xbar[2][0.75]{%
    \sbox{\myboxA}{$\m@th#2$}%
    \setbox\myboxB\null
    \ht\myboxB=\ht\myboxA%
    \dp\myboxB=\dp\myboxA%
    \wd\myboxB=#1\wd\myboxA
    \sbox\myboxB{$\m@th\overline{\copy\myboxB}$}
    \setlength\mylenA{\the\wd\myboxA}
    \addtolength\mylenA{-\the\wd\myboxB}%
    \ifdim\wd\myboxB<\wd\myboxA%
       \rlap{\hskip 0.5\mylenA\usebox\myboxB}{\usebox\myboxA}%
    \else
        \hskip -0.5\mylenA\rlap{\usebox\myboxA}{\hskip 0.5\mylenA\usebox\myboxB}%
    \fi}
\makeatother

\newcommand{\tens}[1]{\boldsymbol{\mathscr{#1}}}

\newcommand{\mat}[1]{\ensuremath{\mathbf{#1}}}

\newcommand{\argmin}{\mathop{\rm argmin}}

\newcommand{\expect}{\mathbb{E}}

\newcommand{\tX}{\tens{X}}

\renewcommand{\S}{\mat{S}}
\newcommand{\I}{\mat{I}}

\newcommand{\U}{\mat{U}}
\newcommand{\V}{\mat{V}}

\newcommand{\X}{\mat{X}}
\newcommand{\R}{\mat{R}}

\newcommand{\G}{\mat{G}}

\renewcommand{\S}{\mat{S}}

\renewcommand{\O}{\mat{O}}
\newcommand{\Q}{\mat{Q}}

\renewcommand{\H}{\mat{H}}

\title{MULTIPAR: Supervised Irregular Tensor Factorization with Multi-task Learning
}

\author{
  Yifei Ren \\
  Microsoft \\
  Redmond, USA\\
  \texttt{yifeiren@microsoft.com} \\
   \And
  Jian Lou \\
  Xidian University \\
  Shenzhen, CN\\
  \texttt{jlou@xidian.edu.cn} \\
     \And
Li Xiong \\
  Emory University \\
  Atlanta, USA\\
  \texttt{lxiong@emory.edu} \\
     \And
 Joyce C Ho \\
  Emory University \\
  Atlanta, USA\\
  \texttt{joyce.c.ho@emory.edu} \\
     \And
  Xiaoqian Jiang \\
  Health Science Center of University of Texas\\
  Houston, USA\\
  \texttt{Xiaoqian.Jiang@uth.tmc.edu} \\
  \And
  Sivasubramanium Bhavan \\
  Emory School of Medicine\\
  Atlanta, USA\\
  \texttt{sivasubramanium.bhavani@emory.edu} \\
}

\begin{document}
\maketitle

\begin{abstract}
Tensor factorization has received increasing interest due to its intrinsic ability to capture latent factors in multi-dimensional data with many applications such as recommender systems and Electronic Health Records (EHR) mining.  PARAFAC2 and its variants have been proposed to address irregular tensors where one of the tensor modes is not aligned, e.g., different users in recommender systems or patients in EHRs may have different length of records. PARAFAC2 has been successfully applied on EHRs for extracting meaningful medical concepts (phenotypes). Despite recent advancements, current models' predictability and interpretability are not satisfactory, which limits its utility for downstream analysis. In this paper, we propose MULTIPAR: a supervised irregular tensor factorization with multi-task learning. MULTIPAR  is flexible to incorporate both  static (e.g. in-hospital mortality prediction) and continuous or dynamic (e.g. the  need for ventilation) tasks. By supervising the tensor factorization with downstream prediction tasks and leveraging  information from  multiple  related predictive  tasks, MULTIPAR can yield not only more meaningful phenotypes but also better predictive performance for downstream tasks. We conduct extensive experiments on two real-world temporal EHR datasets to demonstrate that MULTIPAR is scalable and achieves better tensor fit with more meaningful subgroups and stronger predictive performance compared to existing state-of-the-art methods. 
\end{abstract}


\section{Introduction}
Tensor factorization has received increasing interest due to its intrinsic ability to capture the multi-dimensional structure in the data. It has a wide range of applications including social network analysis \cite{MetaFac2009,link2011}, health data mining \cite{rubik2015,GigaTensor,CoSTCo,taste2019,repair2020,logpar2020,afshar2018copa}, recommender systems \cite{recommender2010}, and signal processing \cite{signal2016}. Canonical Polyadic (CP) \cite{Carroll1970AnalysisOI,Harshman1970FoundationsOT}
, Tucker \cite{tucker}, and tensor singular value decomposition (SVD) \cite{thirdKilmer,factorization2011} are popular regular tensor factorization methods, where each mode of the tensor has fixed size. However, in real-world cases, different people in recommender systems or patients in health data  may have different length of records, which can not be handled by regular tensor factorization methods. PARAFAC2 \cite{Harshman1972Parafac2MA,PARAFAC2part1} has been proposed for factorizing irregular tensors, where one of the mode size is not fixed. We introduce our motivating application for Electronic Health Records (EHRs) mining below for irregular tensor factorization. 

EHRs are patient-centered records collected from a variety of institutions, hospitals, pharmaceutical companies across a long period of time. 
EHR mining can significantly improve the ability to diagnose diseases and reduce or even prevent medical errors, thereby improving patient outcomes \cite{miningyadav2017}. However, directly using the raw, massive, high-dimensional, and longitudinal EHRs is challenging. For example, a single disease can consist of several heterogeneous subgroups yet be coded with the same diagnosis category, e.g., hypertension can be divided into hypertension resolver, hypertension, and prehypertension subgroups. Thus, researchers and healthcare practitioners seek to identify \emph{phenotypes}, or disease subgroups, to better understand differences in biological mechanisms and treatment responses, which can lead to more effective and precise treatment.

Tensor factorization has been widely used to capture the multi-dimensional  structure in EHRs for computational phenotyping \cite{marble2014,limestone2014,rubik2015,GigaTensor,CoSTCo,taste2019,repair2020,logpar2020,afshar2018copa}. Compared to traditional clustering-based approaches, tensor factorization can mine concise and potentially more interpretable latent information between multiple attributes (e.g., diagnosis and medications) in addition to clustering patients into subgroups. 
One key characteristic of EHR data is that different patients  have different visit lengths, varying disease states, and varying time gaps between consecutive visits. Hence PARAFAC2 has been applied to extract pheonotypes from EHR data. 

{\em EHR Example.} 
Table \ref{tbl:raw} shows an EHR database with a set of patient records. Each record consists of a sequence of visits.  Each visit contains measurements for a set of medical features (e.g., laboratory results, medications, etc.). Circles denote missing observations. The number of visits can be different across patients. 

\begin{table}
\centering
\scriptsize
\resizebox{0.5\textwidth}{!}{%
\begin{tabular}{|>{\columncolor[gray]{0.8}} p{10mm}  p{8mm} p{12mm} p{8mm} p{12mm} p{12mm} p{8mm} p{8mm} p{2mm}} 
\toprule 
\textbf{patient 1 \newline visit time}&\textbf{Albumin}&\textbf{Blood urea nitrogen}&\textbf{Chloride}&\textbf{Glascow coma scale} & \textbf{Oxygen saturation}& \textbf{Sodium}& \textbf{......}\\ \midrule 
\color[HTML]{3531FF}0 &\color[HTML]{CB0000}1.8 & \color[HTML]{3531FF}O & \color[HTML]{3531FF}O & \color[HTML]{3531FF}O & \color[HTML]{CB0000}97.5 & \color[HTML]{3531FF}O&\color[HTML]{3531FF}......\\ 
\color[HTML]{3531FF}1& \color[HTML]{3531FF}O & \color[HTML]{3531FF}O & \color[HTML]{3531FF}O & \color[HTML]{CB0000}11 & \color[HTML]{3531FF}O & \color[HTML]{CB0000}135&\color[HTML]{3531FF}......\\ 
\color[HTML]{3531FF}2&\color[HTML]{3531FF}O & \color[HTML]{3531FF}O & \color[HTML]{CB0000}109 & \color[HTML]{3531FF}O & \color[HTML]{3531FF}O & \color[HTML]{3531FF}O&\color[HTML]{3531FF}......\\ 
\color[HTML]{3531FF}3&\color[HTML]{3531FF}O & \color[HTML]{CB0000}24 & \color[HTML]{3531FF}O & \color[HTML]{3531FF}O & \color[HTML]{CB0000}91.4 & \color[HTML]{3531FF}O &\color[HTML]{3531FF}......\\ \midrule
\textbf{patient 2 \newline visit time}&\textbf{Albumin}&\textbf{Blood urea nitrogen}&\textbf{Chloride}&\textbf{Glascow coma scale} & \textbf{Oxygen saturation}& \textbf{Sodium}& \textbf{......}\\ \midrule 
\color[HTML]{3531FF}0 &\color[HTML]{3531FF}O  & \color[HTML]{CB0000}22 & \color[HTML]{3531FF}O & \color[HTML]{3531FF}O & \color[HTML]{3531FF}O & \color[HTML]{3531FF}O&\color[HTML]{3531FF}......\\ 
\color[HTML]{3531FF}1& \color[HTML]{3531FF}O & \color[HTML]{CB0000}18 & \color[HTML]{3531FF}O & \color[HTML]{3531FF}O & \color[HTML]{CB0000}95 & \color[HTML]{3531FF}O&\color[HTML]{3531FF}......\\  \bottomrule
\end{tabular}}
\caption{Sample EHRs data (extracted from MIMIC-EXTRACT dataset, circles means missing observations). }\label{tbl:raw}
\end{table}

\begin{figure}
\centering
\includegraphics[scale=0.28]{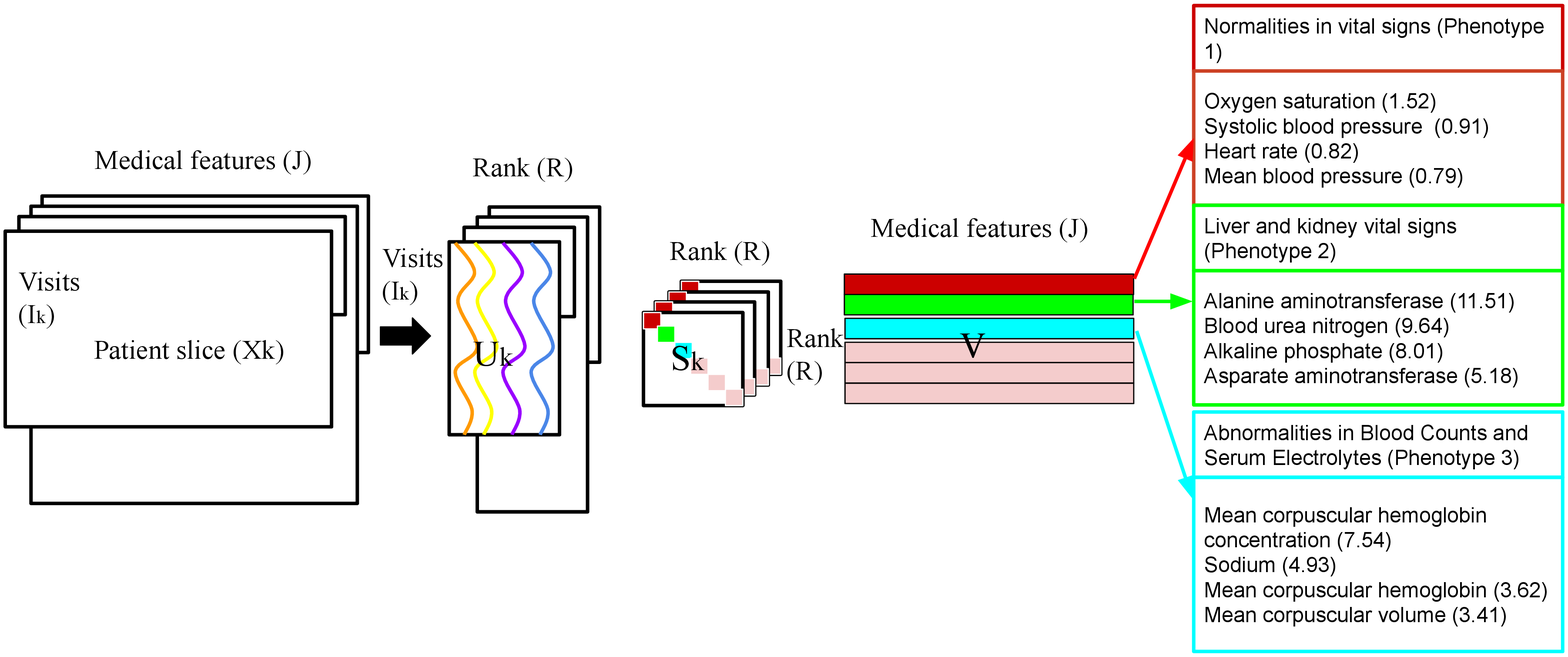}
\vspace{-0.6cm}
\caption{PARAFAC2 tensor factorization}

\label{fig_PARAFAC2}
\end{figure}

Each patient record can be captured using a binary, numeric, or count matrix $X_k$, where each matrix value represents the measurement associated with a particular feature for a particular visit. 
The entire data can be represented as an irregular tensor as shown in Figure \ref{fig_PARAFAC2} where each slice $X_k$ represents the information of patient $k$ with $I_k$ visits and $J$ medical features.
Figure \ref{fig_PARAFAC2} illustrates the computational phenotyping process using PARAFAC2. Each slice of the irregular tensor $X$ will be factorized by PARAFAC2 to three factor matrices. $\U_k \in \mathbb{R}^{I_k\times R} $ captures temporal evolution of the $R$ phenotypes for patient $k$. $\V \in \mathbb{R}^{J\times R}$ contains the $R$ phenotypes. Each row of $V$ matrix represents one latent and potentially interpretable phenotype. Each medical feature is represented with a weight indicating its contribution to the phenotype in each row. $\S_k\in \mathbb{R}^{R\times R}$ is a diagonal matrix with the importance membership of patient $k$ in each one of the $R$ phenotypes.
The right side table in Figure \ref{fig_PARAFAC2} shows three example phenotypes represented in the $V$ matrix with the top 4 highest weighted medical features in each phenotype (the example is from the MIMIC-EXTRACT dataset \cite{mimicextract} and the associated weight is shown in parenthesis). Each phenotype represents a set of related medical features which can suggest meaningful subgroups. 


\begin{figure}
\centering
\includegraphics[scale=0.25]{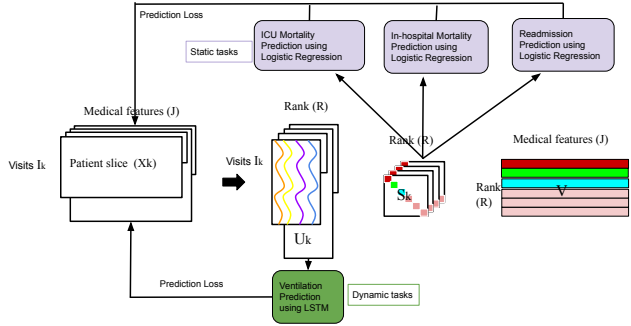}

\caption{Overview of MULTIPAR on MIMIC-EXTRACT dataset}

\label{fig_MULTI}
\end{figure}

Various works have proposed improvements to the basic PARAFAC2 model. SPARTan \cite{perros2017spartan} modified the MTTKRP calculation order of PARAFAC2 to handle large and sparse data. COPA \cite{afshar2018copa} further introduced various constraints to improve the interpretability of the factor matrices for more meaningful phenotype extraction. REPAIR \cite{repair2020} and LogPar \cite{logpar2020} added a low-rankness constraint to improve PARAFAC2 robustness to missing data. Despite these advances, current PARAFAC2 models are completely unsupervised and only attempt to learn the latent factors to best recover  the  original  observations. Some works have considered using the latent factors as features for downstream prediction tasks (e.g., in-hospital mortality or hospital readmission prediction using extracted phenotypes), and achieved limited performance gain than using the raw data as features. This is because the  tensor factorization does not take advantage of the downstream labels. The extracted factors, while interpretable, may not be the most  representative or discriminating for downstream prediction tasks. 
In addition, current work \cite{afshar2018copa,logpar2020,repair2020,perros2017spartan,Xiaoqian2017Discriminative} using tensor factorization for predictive tasks only consider a single task (e.g., in-hospital mortality prediction) and ignore useful information from other prediction tasks. 

To address these limitations, we propose MULTIPAR: a supervised irregular tensor factorization with multi-task learning for both phenotype extraction and predictive learning, as shown in Figure \ref{fig_MULTI}.  
MULTIPAR jointly optimizes the tensor factorization and downstream prediction together, so that the factorization can be ``supervised" or informed by the predictive tasks. In addition, we use a multi-task framework to leverage information from multiple predictive tasks. It provides flexibility to incorporate both one-time or static (e.g. in-hospital mortality prediction) and continuously changing or dynamic (e.g. the need for ventilation) outcomes. To achieve this, the temporal features from $\U$ matrix are used for dynamic prediction and the features from $\S$ matrix are used for static prediction, as shown in Figure \ref{fig_MULTI}. 

Our main hypothesis is that such a supervised multi-task framework can yield not only more meaningful phenotypes but also better predictive accuracy than performing tensor factorization independently followed by predictive learning using the phenotypes extracted from the tensor. In addition, by sharing the representation across tasks, the learned phenotypes can generalize better for each task. Our empirical studies on two large publicly available EHR datasets with representative predictive tasks (both static and dynamic) and different models (e.g. logistic regression and recurrent neural networks) verified this hypothesis.

In summary, we list our main contributions below:
\begin{enumerate}
    \item We propose a supervised framework for PARAFAC2 tensor factorization and downstream prediction tasks such that the factorization can be ``supervised" or informed by the predictive tasks. 

    \item We use a multi-task framework to leverage information from multiple predictive tasks and provide flexibility to incorporate both static and  dynamic tasks and different models (e.g. logistic regression and recurrent neural networks). 
    \item We introduce a novel unified and dynamic weight selection method for weighing the tensor factorization and predictive tasks during the optimization process, where the tensor factorization is considered as one task, to achieve an overall optimized result. 
    \item We evaluate MULTIPAR's tensor reconstruction quality, predictability, scalability, and interpretability on two real-world temporal EHR datasets through a set of experiments, which verify MULTIPAR can identify more meaningful subgroups and yield stronger predictive performance compared to existing state-of-the-art approaches. 
\end{enumerate}

\section{Background}

\subsection{Irregular Tensor Factorization}
 PARAFAC2 model is the state-of-the-art tensor factorization framework for  irregular tensor, i.e., tensors that do not align along one of its modes. 
The PARAFAC2 model decomposes each frontal slice of the irregular tensor $\X$ as $\mat{U}_k\S_k\V^{\top}$, where $\mat{U}_k \in \mathbb{R}^{I_k\times R}$, $\S_k\in \mathbb{R}^{R\times R}$ is diagonal and $\V \in \mathbb{R}^{J\times R}$. $R$ is the target rank. Uniqueness is an important property in tensor factorization models that ensures the solution is not an arbitrarily rotated version of the actual latent factors. In order to enforce uniqueness, Harshman \cite{Harshman1970FoundationsOT} imposed the constraint $\mat{U}^T\mat{U}_k = \Phi, \forall k$. This is equivalent to each $\mat{U}_k$ being decomposed as $\mat{U}_k = \Q_K\H$, where $\Q_k\in\R^{I_k\times R}$, $\Q_k^{\top}\Q_k = \I$, and $\H \in\mathbb{R}^{R\times R}$. Given the above modeling, the standard algorithm to fit PARAFAC2 solves the following optimization problem:

\begin{definition} (Original PARAFAC2 model)
\label{def.parafac2.model}
$$ \argmin _{\{\mat{U}_k\},\{\S_k\},\V} \sum _{k=1}^K\frac{1}{2} \|\X_k-\mat{U}_k\S_k\V^{\top}\|_F^2,$$
    subject to: \ $\mat{U}_k = \Q_k\H, \Q_k^{\top}\Q_k = \I, \S_k \text{ is diagonal}$.
\end{definition}

Given a tensor representing the EHR data as in Figure \ref{fig_PARAFAC2} where each slice $X_k$ represents the information of patient $k$ with $I_k$ visits and $J$ medical features, PARAFAC2 decomposes the irregular tensor $\X$ into the factorization matrices which have the following  interpretations:

\begin{itemize}
    \item $\U_k \in \mathbb{R}^{I_k\times R} $ represents the temporal trajectory of $I_k$ clinical visits in each one of the $R$ phenotypes. 
    
    \item $\V \in \mathbb{R}^{J\times R}$ represents the relationship between medical features and phenotypes. 

    \item $\S_k\in \mathbb{R}^{R\times R}$ represents the relationship between patients and phenotypes. Each column in S represents one phenotype, and if a patient has the highest weight in a specific phenotype, it means the patient is mostly associated with or exhibits a particular phenotype.

\end{itemize}
SPARTan \cite{perros2017spartan} was developed to decompose large sparse datasets, and COPA \cite{afshar2018copa} extended SPARTan to enhance the interpretability by adding more constraints, e.g., smoothness on $\U_k$ and sparsity on $\V$. REPAIR \cite{repair2020} and LogPar \cite{logpar2020} added low-rankness constraints to improve the robustness of PARAFAC2 model to handle missing values. However, no current work has considered using downstream tasks to supervise and improve the predictability of PARAFAC2 model as Table \ref{tbl:improvement} shows. We note that while we focus on introducing the predictability of MULTIPAR in this paper, the framework can be easily extended to incorporate robustness by adding low-rankness constraint on factor matrices.

\begin{table}
\centering
\resizebox{0.5\textwidth}{!}{%
\begin{tabular}{l l l l l} \toprule
\textbf{PARAFAC2 model}&\textbf{Scalability}&\textbf{Predictability}&\textbf{Robustness}&\textbf{Interpretability}\\ 
\midrule
Original PARAFAC2 & x & x & x & x\\ 
SPARTan \cite{perros2017spartan}& \checkmark & x & x & x\\ 
COPA \cite{afshar2018copa} & \checkmark & x & x & \checkmark\\
REPAIR \cite{repair2020} & \checkmark & x & \checkmark & \checkmark \\ 
LogPar \cite{logpar2020} & \checkmark & x & \checkmark & x \\ 
MULTIPAR & \checkmark & \checkmark & x & \checkmark \\ \bottomrule
\end{tabular}}
\caption{Comparison of existing PARAFAC2-based models }\label{tbl:improvement}
\end{table}




\subsection{Supervised and Multi-task learning Framework}
Supervision can enhance traditional unsupervised tasks such as clustering for a wide range of applications, e.g., graph learning \cite{jointwang}, pattern classification \cite{jointliu}. 
Wang et al. proposed a supervised feature extraction framework using discriminative clustering to improve graph learning model's clustering accuracy \cite{jointwang}. Liu te al. proposed a supervised minimum similarity projection framework using lowest correlation representation to improve pattern model's classification accuracy \cite{jointliu}. 


Over the past years, multi-task learning (MTL) \cite{Zhang2018overview, Zhang2017survey} has attracted much attention in the artificial intelligence and machine learning communities. Traditional machine learning frameworks solve a single learning task each time, which ignores commonalities and differences across different tasks. MTL  aims to learn multiple related tasks jointly so that the knowledge contained in one task can be leveraged by other tasks, with the hope of improving generalization performance by learning a shared representation
\cite{Baxter2000, Thrun1999}.
MTL has been used successfully across all applications, from natural language processing \cite{Collobert2008, multilabel2009} and speech recognition \cite{Deng2013,chen2015multitask} to computer vision \cite{Girshick2015FastR} and drug discovery \cite{Ramsundar2015drug, multitask2019}. However, no current work has considered improving predictability of tensor factorization using MTL.

 \section{Proposed Method}
In this section, we present the MULTIPAR model in the context of EHR pheynotyping and its optimization. The general framework is applicable to any irregular tensor factorization and predictive learning tasks such as recommender systems. 
\subsection{Problem Formulation}
We formalize the objective function for the MULTIPAR model in Definition \ref{def.MULTIPAR.model}. The PARAFAC2 loss for $\tX$ ensures the reconstructed tensor closely approximates the original tensor. The static outcomes loss and dynamic outcomes loss are separate prediction tasks. Static outcome prediction tasks have a one-time or static labels, and dynamic outcome prediction tasks have a continuously changing or temporal dynamic labels for each time stamp. An approximate uniqueness constraint ensures tensor factorization uniqueness. For EHRs phenotype discovery, various constraints can be imposed on the factorization matrices to yield meaningful and high-interpretability phenotypes. The MULTIPAR model accommodates such interpretability-purposed constraints in eq. (\ref{eq.MULTIPAR.model.formal})  including:  non-negativity for $c_1(\S_k)$, sparsity for $c_2(\V)$. We explain each of the loss components and constraints in detail below. 
\begin{definition} (MULTIPAR objective function)
\label{def.MULTIPAR.model}
  \begin{equation}
  \label{eq.MULTIPAR.model.formal}
    \begin{split}
  &\argmin_{\Q_k,\H,\S_k,\V}\sum_{k=1}^K\sum _{(i,j)\in \Omega} \overbrace{\rho _1L_1(\X_{ijk}, \{\U_k\S_k\V^{\top}\}_{ijk})}^{\text{PARAFAC2 loss for } \tX}\\&+ \overbrace{\rho _2L_2({\S_k})}^{\text{static outcomes loss}}
  + \overbrace{\rho _3L_3({\U_k})}^{\text{dynamic outcomes loss}}\\&
  + \overbrace{\varrho_1\|\U_k^{\top}\U_k - \I\|_F^2\big{)}}^{\text{approximate uniqueness constraint}}
  + \overbrace{\sum _{k=1}^K c_1(\S _k)}^{\text{non-negativity constraint}} + \overbrace{c_2\|\V\|_1}^{\text{sparsity constraint}}
  \end{split}
\end{equation}
where $k = 1,...,K$, $\H, \{\S_k\}, \I \in \mathbb{R}^{R \times R}$. $c_1$ is the nonnegativity constraint, and $c_2\|\V\|_1$ is the sparsity penalty.
\end{definition}

\partitle{PARAFAC2 loss}
The PARAFAC2 tensor factorization loss can ensure the reconstructed tensor closely approximate the original tensor. To accommodate different data types, the PARAFAC2 loss can be any smooth loss function, e.g., Least square loss, Poisson loss \cite{Hong_2020} and Rayleigh Loss \cite{Hong_2020}.

\partitle{Static outcomes loss}
Previous PARAFAC2 models separate the PARAFAC2 training process and downstream prediction process. For example, in-hospital mortality prediction accuracy may be used as the metric to measure the  predictability of the phenotypes extracted by the model. In the MULTIPAR model, we optimize the downstream prediction tasks and tensor factorization together by adding the prediction losses of the prediction tasks to the objective function.
If the prediction task has one label per patient, we denote it as a static outcome prediction task. For illustrative purposes, we use a logistic regression model on the $\S$ matrix to predict static outcome tasks, and add the cross-entropy loss to the objective function. In fact, any differentiable loss function (e.g., square loss, exponential loss) can be incorporated in the objective function.

\partitle{Dynamic outcomes loss}
Different from static outcomes, dynamic outcomes have labels at each timestamp. For example, predicting whether a patient will be on a ventilator at a given future time can be used to measure the model's predictability. For illustrative purposes, we use the long short-term memory (LSTM) model on the $\U$ matrix to predict each patient's dynamic outcome labels, and add the loss of the LSTM model to the objective function. Similar to the static outcome loss, other models (e.g., gated recurrent units, vanilla recurrent neural networks) and their associated loss functions can be incorporated in the objective function.

\partitle{Approximate uniqueness constraint}
The optimization of the original PARAFAC2 model adopts the alternating direction method of multipliers (AO-ADMM) \cite{AOADMM}, which can not make full use of the parallel computation feature of GPUs. To adopt mainstream deep learning frameworks like Pytorch and Tensorflow, we use a stochastic gradient descent (SGD) based optimization approach. The uniqueness constraint in the original PARAFAC2 model is $\Q_k^{\top} \Q_k = \I$. Similar to LogPar \cite{logpar2020}, to optimize $\Q_k$ we relax the uniqueness constraint to $\|\Q_k^{\top} \Q_k - \I\|_F^2$.

\partitle{Sparsity on $\V$}
The $\V$ matrix captures the association between a medical feature and a particular phenotype. In order to improve interpretability, we introduce a sparsity constraint on the $\V$ matrix. $l_0$ and $l_1$ norms are two popular sparsity regularization techniques. The $l_0$ regularization norm, also relaxed by hard thresholding, will cap the number of non-zero values in a matrix. The $l_1$ regularization norm, also relaxed by soft thresholding, will shrink matrix values towards zero. As hard thresholding is a non-convex optimization problem which can not be optimized by the SGD framework, we adopt the soft thresholding, which is convex and can be migrated into the SGD optimization framework.

\partitle{Non-negativity on $\S$}
The diagonal matrix $\S$ indicates the importance membership of patient $k$ in each one of the $R$ phenotypes. Since only non-negative membership values makes sense, we zero out the negative values in $\S$. 

\partitle{SDW: Smooth dynamic weight selection}
Numerous deep learning applications benefit from MTL with multiple regression and classification objectives.
Yet the performance of MTL is strongly dependent on the
relative weighting between each task’s loss. Our objective function consists of several losses from the tensor factorization and the predictive tasks. Each of this loss is associated with a weight. While Definition \ref{def.MULTIPAR.model} shows $\rho _1$ as the weight for tensor loss, $\rho _2$ and $\rho _3$ as the accumulative weights for static and dynamic tasks respectively, here we use $Weight_n(t)$ to denote the weight for each individual task $n$ in epoch $t$. 
A key challenge is how to tune these weights for different tasks. 

While the DWA weight selection \cite{MTA2018} was proposed to dynamically change the task weights at each epoch by considering the rate of change of the loss over the epoch, the noisy nature of SGD weights can cause drastic fluctuations in the task weights between epochs. This can cause oscillating behavior between the various tasks and impedes convergence of the algorithm. Therefore, we propose a novel smooth dynamic weight selection method to choose the weight for each task.
We first calculate the relative descending rate of each task loss and denote it as $\omega_n(t-1)$. $t$ here represents an epoch index:

\begin{equation}
\label{eq.weightloss}
     \omega_n(t-1) = \frac{Loss_n(t-1)}{Loss_n(t-2)}.
\end{equation}

We then calculate the weight for each task using the following equation:
\begin{equation}
\label{eq.weightupdate}
    Weight_n(t) := N\frac{exp(\sum_{j=t-1}^m(\omega_n(t-j)/C)/m)}{\sum_{i=1}^N{exp(\sum_{j=t-m}^m(\omega_i(t-j)/C)/m)}}. 
\end{equation}

Intuitively, each task weight is dynamically updated based on a "smoothed" descending rate of the loss, the higher the rate (i.e. the more the task contributes to the optimization objective), the higher the weight for the task in next epoch. 
We use $C$ to control the softness distribution between different tasks. If $C$ is large enough, the weight for each task will be uniform. Different from \cite{hinton2015distilling}, we introduce $m$, the weight update window size. The task weights are updated based on an average of the descending rate of the loss over $m$ past epochs from $t-m$ to $t-1$ (instead of using one previous iteration). The main rationale for this smoothing is to reduce the SGD update uncertainty and training data selection randomness. Finally, a softmax operator, which is multiplied by the number of tasks $N$, ensures the sum of the weight equals $N$. For $t=1$, we initialize all the weights to 1.



\subsection{Optimization}
To solve the optimization problem in Eq. \eqref{eq.MULTIPAR.model.formal}, MULTIPAR follows an alternative optimization strategy where we optimize one variable individually with all other variables fixed. According to the subproblem differentiable, we group the variables into two groups: pure smooth subproblems which can be directly solved by SGD and proximal mapping-based smooth subproblems \cite{parikh2014proximal}. Proximal map is a key building block for optimizing nonsmooth regularized objective functions, e.g., the $\|\cdot\|_1$ $\ell_1$-norm regularization function for inducing sparsity and $\|\cdot\|_*$ nuclear norm regularization for inducing low-rankness. In the following, we omit the iteration number for brevity in notation.

\subsubsection{Pure Smooth Subproblems Updates.}
For the pure smooth subproblems, we use SGD to update the variables, which include the following three parts:

\partitle{Update of $\U_k$} The subproblem of $\U_k$ takes the form as follows
\begin{equation}
\label{eq.update.U}
\begin{split}
    & \arg\min_{\U_k} \sum _{(i,j)\in \Omega} \rho _1L(\X_{ijk}, \{\U_k\S_k\V^{\top}\}_{ijk}) \\ & + \varrho_1 \|\U_k-\Q_k\H\|_F^2 + \rho _3L_3({\U_k}).
\end{split}
\end{equation}

\begin{algorithm}[t]
\caption{Optimization Framework for MULTIPAR}
\label{optimiztionalgorithm}
\begin{algorithmic}[1]
\REQUIRE Input tensor $\tX$; Model parameters $\rho_1$-$\rho _3$, $\varrho_1$-$\varrho_2$; Interpretability constraint types $c_1,c_2$; Initial rank estimation $R$.
\WHILE{Not reach convergence criteria}
\STATE Update $\{\U_k\}$ using eq.(\ref{eq.update.U}) by SGD;
\STATE Update $\{\Q_k\}$ using eq.(\ref{eq.update.Q}) by SGD; 
\STATE Update $\H$ using eq.(\ref{eq.update.H}) by SGD;
\STATE Update $\S_k$ using eq.(\ref{eq.update.S}) by Proximal/Projected SGD;
\STATE Update $\V$ using eq.(\ref{eq.update.V}) by Proximal/Projected SGD;
\STATE Calculate weight for each prediction task using eq. \ref{eq.weightloss} and eq.\ref{eq.weightupdate} by SDW;
\ENDWHILE
\ENSURE Phenotype factor matrices $\{\U_k\}=\{\Q_k\H\}, \{\S_k\}, \V$.
\end{algorithmic}
\end{algorithm}

\partitle{Update of $\Q_k$} The subproblem of $\Q_k$ takes the form as follows
\begin{equation}
\label{eq.update.Q}
    \arg\min_{\Q_k} \varrho_1 \|\U_k-\Q_k\H\|_F^2 + \varrho_2\|\Q_k^{\top}\Q_k - \I\|_F^2.
\end{equation}

\partitle{Update of $\H$} The subproblem of $\H$ takes the form as follows
\begin{equation}
\label{eq.update.H}
    \arg\min_{\H} \sum_{k=1}^K\|\U_k-\Q_k\H\}\|_F^2 
\end{equation}

\subsubsection{Proximal Mapping-base Smooth Subproblems Updates}
For the nonsmooth subproblems, we propose a proximal mapping-based \cite{parikh2014proximal} update, which include the following two parts. 

\partitle{Update of $\S_k$} The subproblem of $\S_k$ takes the form as follows
\begin{equation}
\label{eq.update.S}
    \arg\min_{\S_k} \sum _{(i,j)\in \Omega}\rho _1 L(\X_{ijk}, \{\U_k\S_k\V^{\top}\}_{ijk}) + \rho _2L_2({\S_k}) + c_1(\S_k). 
\end{equation}
We use projected SGD to update $\S_k$, where each step takes the following form
\begin{equation}
    \S_k = \max(0,\S-\lambda \G[\S_k]),
\end{equation}
where $G[\S_k]$ denotes the stochastic gradient of the smooth part $\sum _{(i,j)\in \Omega} \rho _1 L(\X_{ijk}, \{\U_k\S_k\V^{\top}\}_{ijk}) + \rho _1 L_2({\S_k})$ with respect to $\S_k$.

\partitle{Update of $\V$} The subproblem of $\V$ takes the form as follows
\begin{equation}
\label{eq.update.V}
\begin{split}
        \arg\min_{\V} \sum_{k=1}^K\sum _{(i,j)\in \Omega} \rho _1L(\X_{ijk}, \{\U_k\S_k\V^{\top}\}_{ijk})
        + c_2 \|\V\|_1. 
\end{split}
\end{equation}
We use soft-thresholding operator to update $\V$, where each step takes the following form: 
${\tt soft-thresholding}(\V-\lambda \G[\V]) = sign((\V-\lambda \G[\V]))((\V-\lambda \G[\V])-\frac{c_2}{\lambda})$, 
where $\lambda$ is the step-size and $G[\V]$ denotes the stochastic gradient of the smooth part\\
$\sum_{k=1}^K\sum _{(i,j)\in \Omega}\rho _1 L(\X_{ijk}, \{\U_k\S_k\V^{\top}\}_{ijk})$ with respect to $\V$.

\partitle{The complete algorithm} The optimization procedure is summarized in Algorithm \ref{optimiztionalgorithm}.

\subsection{Complexity and Convergence Analysis}
The following theorem summarizes the computational complexity of Algorithm \ref{optimiztionalgorithm}.
\begin{theorem} (Per-iteration computational complexity of MULTIPAR algorithm)
For an input tensor $\O_k:\mathbb{R}^{I_k\times J},\ for\ k=1,...,K$ and initial target rank estimation $R$, Algorithm \ref{optimiztionalgorithm}'s per-iteration complexity is $\mathcal{O}(3R^2JK)$.
\end{theorem}
\begin{proof}
MULTIPAR's per-iteration complexity breaks down as follows: Line 2 costs $\mathcal{O}(R(R+J+K))$; Line 3 costs $\mathcal{O}(\min\{R^2 I,RI^2\})$, where $I$ denotes the maximum among $\{I_k\}$; Line 4,5,6 cost $\mathcal{O}(R^2(R+J+K))$. As a result, the per-iteration complexity is $\mathcal{O}(3R^2JK)$.
\end{proof}

\begin{theorem} (Convergence Rate of MULTIPAR algorithm)
Let $\Phi[t] = (\Q[t],\H[t],\S[t],\V[t])$ be the iterates of MULTIPAR at iteration $t$, and $\mathcal{L} = L_1+L_2+L_3$ be the loss function of MULTIPAR. Under certain conditions, MULTIPAR will converge in terms of $\lim_{t\mapsto +\infty} \expect[\|\mathcal{L}(\nabla \mathcal{L}(\Phi [t]))\|] = 0$.
\begin{proof} (Sketch of proof)
MULTIPAR is a multi-block SGD algorithm for nonconvex optimization in nature, which has been extensively studied in optimization literature. The above theorem follows the convergence results established in \cite{xu2015block}.
\end{proof}

\end{theorem}

\section{Experiment}
\subsection{Dataset}
We use two real-world datasets to evaluate MULTIPAR in terms of its reconstruction quality, predictive performance, interpretability, and scalability.

\noindent\textbf{eICU \footnote{\url{https://eicu-crd.mit.edu}} \cite{eICU}:} 
The eICU Collaborative Research Database is a freely available multi-center database for critical care research. It contains variables used to calculate the Acute Physiology Score (APS) III for patients. We select $202$ diagnosis codes that have the highest frequency, as in \cite{KimDiscriminative}. The resulting number of unique ICU visits is 145426. The maximum number of observations for a patient is 215. We select three static outcome prediction tasks, including intubated prediction, ventilation prediction, and dialysis prediction. The ventilation prediction here is a static prediction task indicating whether a patient needs to be ventilated at the time of the worst respiratory rate, we will use "vent-res" as the name for this task. 

\noindent\textbf{MIMIC-EXTRACT \footnote{\url{https://github.com/MLforHealth/MIMIC_Extract/}} \cite{mimicextract}:} 
MIMIC-Extract is an open-source pipeline for transforming raw EHR data in MIMIC-III into data frames that are directly usable in common machine learning pipelines. We use the vitals labs mean table, which contains 34,472 patients with 104 features (Vital lab codes). The maximum number of observations for a patient is 240. We further normalize the data to [0,1]. We select three static outcome prediction tasks, including in-hospital mortality prediction, readmission prediction, ICU mortality prediction, and one dynamic outcome prediction task, which is ventilation prediction for every visit.


\subsection{Evaluation Metrics}
In order to test the tensor reconstruction quality of MULTIPAR model, we adopt the $FIT\in(-\infty,1]$ score \cite{BroPARA} as the quality measure (the higher the better):
\begin{equation}
  FIT = 1- \frac{\sum_{k=1}^{K}\|\X_k-\U_k\S_k\V^T \| ^2}{\sum_{k=1}^{K}\|\X_k \| ^2}.
\end{equation}
The original tensor, denoted as $\{\X_{k}\}$, serves as the ground truth. $\U_k,\S,\V$ are factor matrices after the MULTIPAR tensor factorization. We evaluate the derived phenotypes' predictability power using the PR-AUC score of the prediction tasks. We split the data with a proportion of 8:2 as training and test sets and use PR-AUC score to evaluate the predictive power.

\subsection{Methods for Comparison}
We compare MULTIPAR with three baseline methods: SPARTan, COPA, and  singlePAR. SPARTan and COPA are two state-of-the-art irregular tensor factorization methods. We also compare against a supervised single task PARAFAC2, which is a single-task version of MULTIPAR.

\begin{itemize}
\item \textbf{SPARTan \cite{perros2017spartan} - scalable PARAFAC2}: A tensor factorization method for fitting large and sparse irregular tensor data. It only considers the tensor reconstruction loss.
\item \textbf{COPA \cite{afshar2018copa} - scalable PARAFAC2 with additional regularizations}: An irregular tensor factorization method that introduces various constraints/regularizations to improve the interpretability of the factor matrices. For both SPARTan and COPA, the extracted phenotypes are used for training the models for the downstream predictive tasks. 

\item \textbf{SinglePAR - supervised single task PARAFAC2}: The supervised irregular tensor factorization with single prediction task (single task version of MULTIPAR). The weight of tensor factorization and prediction tasks are also tuned using SDW.
\end{itemize}

\begin{figure}
\centering\includegraphics[scale=0.3]{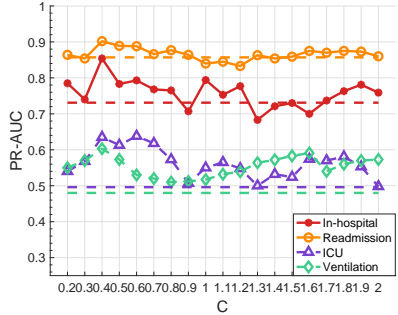}
\caption{PR-AUC score using different C}
\label{fig:Tchoose}
\end{figure}

\subsection{Implementation Details}
The multi-task dynamic weight selection of MULTIPAR has two hyper-parameters that need to be tuned. $C$ controls the softness distribution between different tasks, and $m$ is the weight update window size. In order to find the best $C$, we vary $C$ from $0.2$ to $2$, and compare the prediction tasks' PR-AUC scores on different data set under different ranks. Figure \ref{fig:Tchoose} shows the MIMIC-EXTRACT dataset result when rank $=50$ and $m$ is fixed to $5$. The dashed line is the PR-AUC when all the task have equal weight. $C$ is set to $\frac{1}{\sqrt{N}}$ in the remaining experiments to achieve the best result.

We vary the weight update window size $m$ from $1$ to $10$, and compare the convergence speed and PR-AUC score. We fix $C = \frac{1}{\sqrt{N}}$, and plot the tensor loss in each epoch and set the maximum number of epochs to be 200. 
When $m=1$, it does not converge after 200 epochs. When $m=5$, it requires the least number of epochs to converge (when the total loss plateaus). Although when $m=3$, some prediction tasks' PR-AUC scores are slightly better than $m=5$, it requires too many epochs to converge. Thus, in our experiments below, we adopt $m=5$.

\begin{table}
\centering
\resizebox{0.5\textwidth}{!}{%
\begin{tabular}{l  l l l l l} \toprule 
&\textbf{m=1}&\textbf{m=3}&\textbf{m=5}&\textbf{m=8}&\textbf{m=10}\\ \midrule 
In-hospital mortality prediction task &0.740 & 0.789 & 0.854 & 0.783 & 0.768 \\ 
Readmission prediction task& 0.872 & 0.893 & 0.902 & 0.892 & 0.853\\ 
ICU mortality prediction task&0.626 & 0.638 & 0.635 & 0.583 & 0.571\\ 
Ventilation prediction task&0.600 & 0.605 & 0.603 & 0.591 & 0.587 \\ 
Convergence epoch & 200  & 187 & 98  & 110 & 150 \\   \bottomrule
\end{tabular}}
\caption{Experiment result of PR-AUC and convergence epochs when $m$ varies }\label{tbl:temp}
\end{table}

\subsection{Experiment Result}
\partitle{Tensor reconstruction quality analysis}
For the following experiments on tensor reconstruction quality, we run each method for 5 different random initializations and report the average $FIT$. In addition, we evaluate model completion performance under different target ranks, $R$, from 10 to 60, and run 200 epochs.

First, we compare MULTIPAR's FIT with the baseline models on two datasets shown in Figure \ref{fig:fit}. MULTIPAR optimizes all prediction tasks and tensor factorization together. SPARTan and COPA first finish the tensor factorization, and then predict the downstream prediction tasks. SinglePAR optimizes the evaluated task and tensor factorization together. As Figure \ref{fig:fitextract} and \ref{fig:fiteICU} shows, MULTIPAR outperforms all baseline methods on all datasets. In particular, MULTIPAR achieves a FIT score of 0.97 and 0.71 on MIMIC-EXTRACT and eICU respectively, a 13\% and 40\% relative improvement when compared to the best baseline model SinglePAR, which shows the strong tensor reconstruction ability of MULTIPAR, thanks to the ``supervision" of the multiple predictive tasks. COPA performs better than SPARTan because it introduces various regularizations on the factor matrices, which can slightly improve the tensor reconstruction ability.

SinglePAR performs better than SPARTan and COPA on most of the ranks but is left behind COPA on large ranks. SinglePAR jointly optimizes prediction task and tensor factorization together. We can see that certain tasks benefit the tensor FIT while others may guide the tensor factorization into a suboptimal direction and degrade the tensor reconstruction quality. Although MULTIPAR model is supervised, thanks to the MTL, it can use all of the available outcomes across the different tasks to learn generalized representations of the data that are useful for tensor reconstruction.

\begin{figure}
    \centering
    \begin{subfigure}{0.3\textwidth}
        \includegraphics[width=\textwidth]{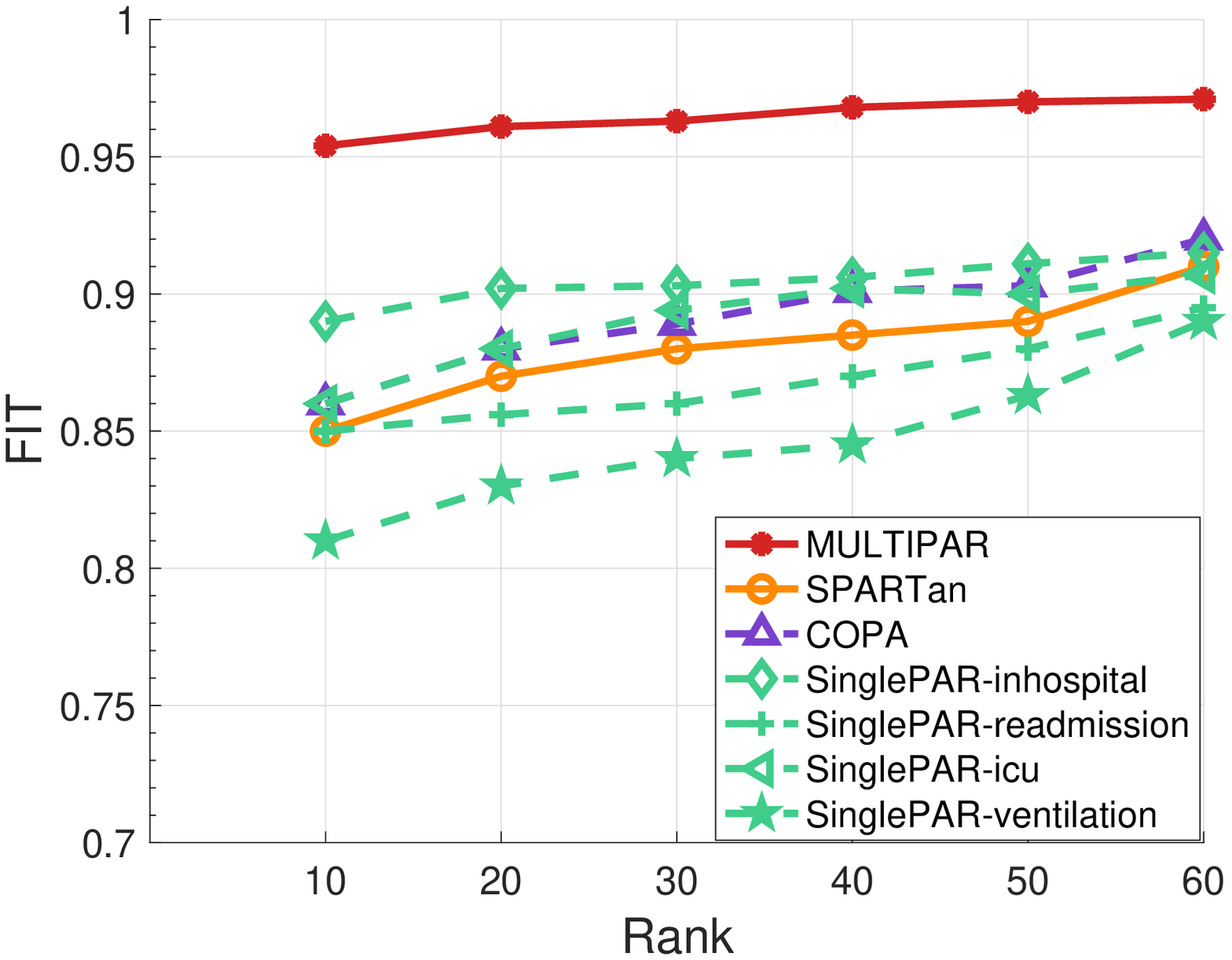}
        \caption{MIMIC-EXTRACT}
        \label{fig:fitextract}
    \end{subfigure}
     \begin{subfigure}{0.3\textwidth}
        \includegraphics[width=\textwidth]{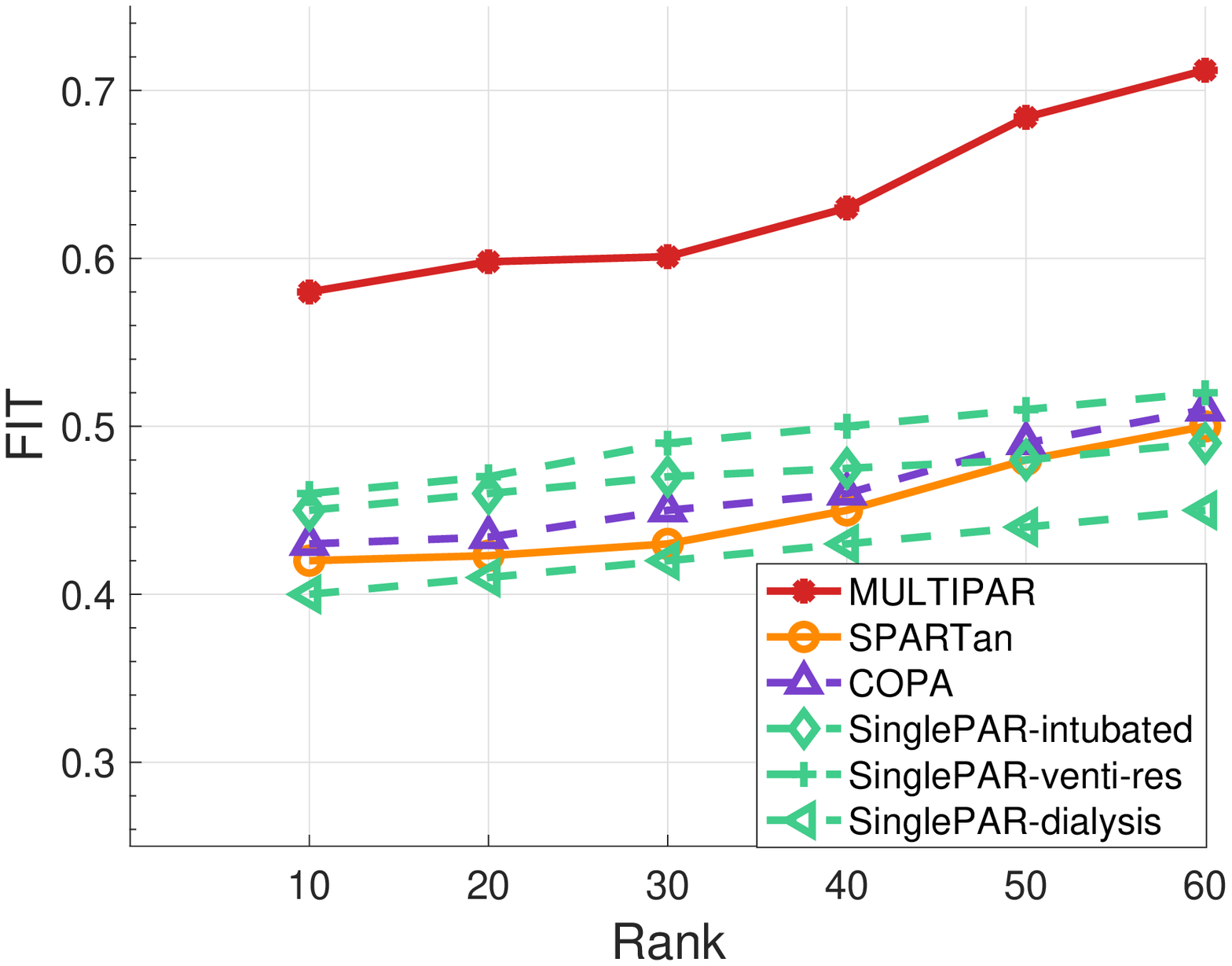}
        \caption{eICU}
        \label{fig:fiteICU}
    \end{subfigure}
     \caption{FIT score on MIMIC-EXTRACT and eICU dataset.}\label{fig:fit}
\end{figure}  

\begin{figure}
    \centering
    \begin{subfigure}{0.3\textwidth}
        \includegraphics[width=\textwidth]{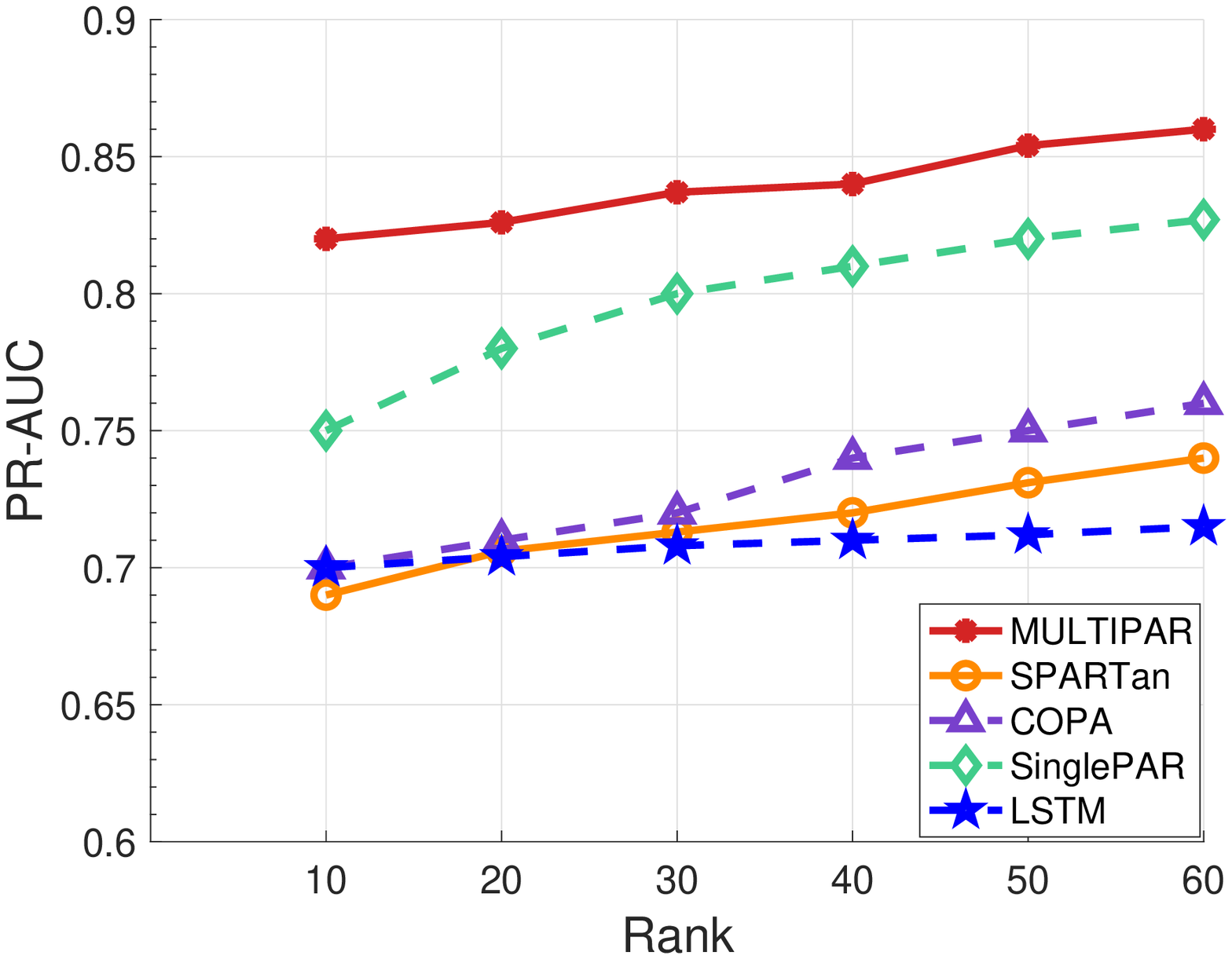}
        \caption{In-hospital Mortality Prediction}
        \label{fig:aucinhos}
    \end{subfigure}
     \begin{subfigure}{0.3\textwidth}
        \includegraphics[width=\textwidth]{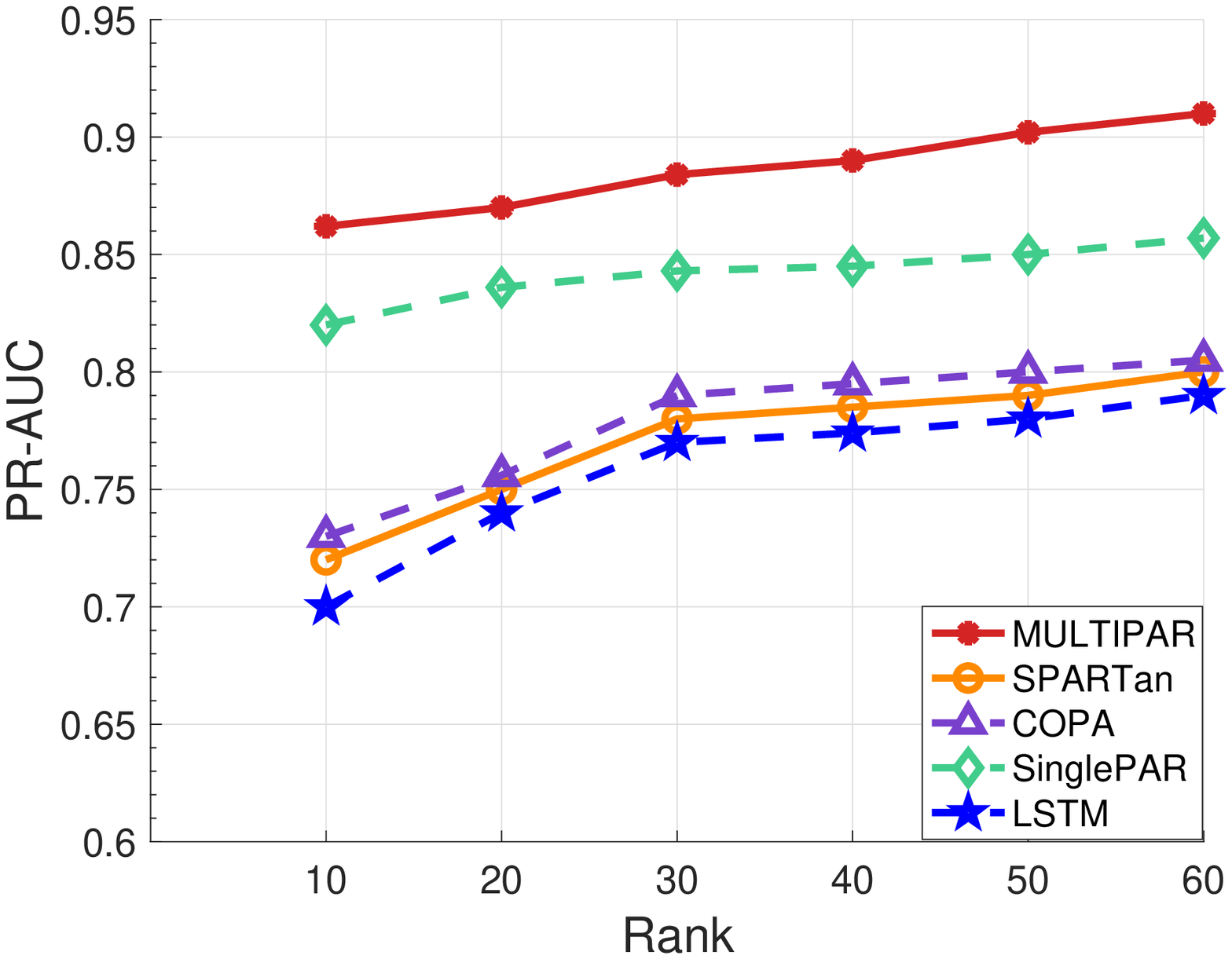}
        \caption{Readmission Prediction}
        \label{fig:aucreadmission}
    \end{subfigure}\\
    \begin{subfigure}{0.3\textwidth}
        \includegraphics[width=\textwidth]{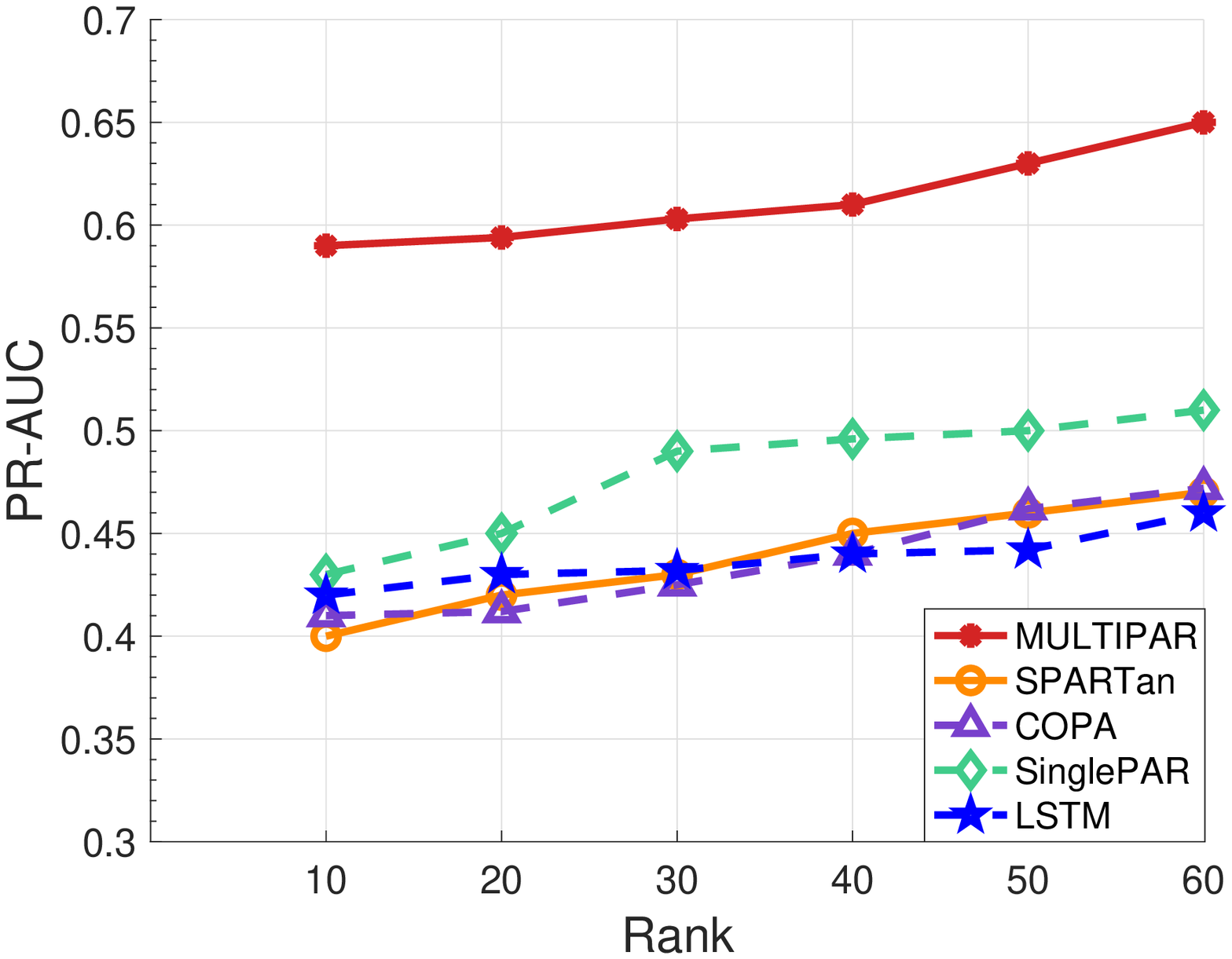}
        \caption{ICU Mortality Prediction}
        \label{fig:aucicu}
    \end{subfigure}
        \begin{subfigure}{0.3\textwidth}
        \includegraphics[width=\textwidth]{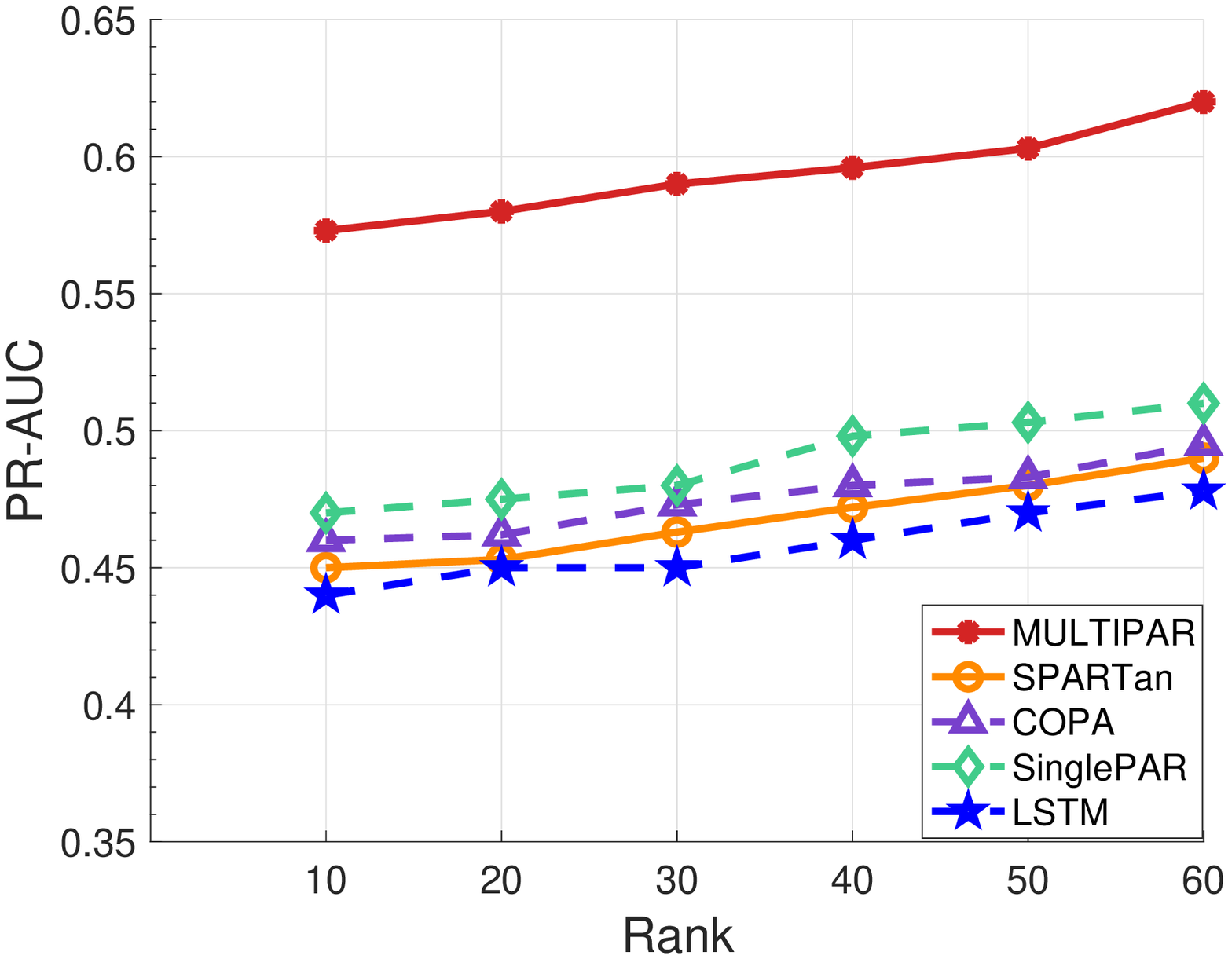}
        \caption{Ventilation Prediction}
        \label{fig:aucventi}
    \end{subfigure}
        \caption{PR-AUC for prediction tasks on MIMIC-EXTRACT}\label{fig:aucEXTRACT}
\end{figure}

\begin{figure}
    \centering
    \begin{subfigure}{0.3\textwidth}
        \includegraphics[width=\textwidth]{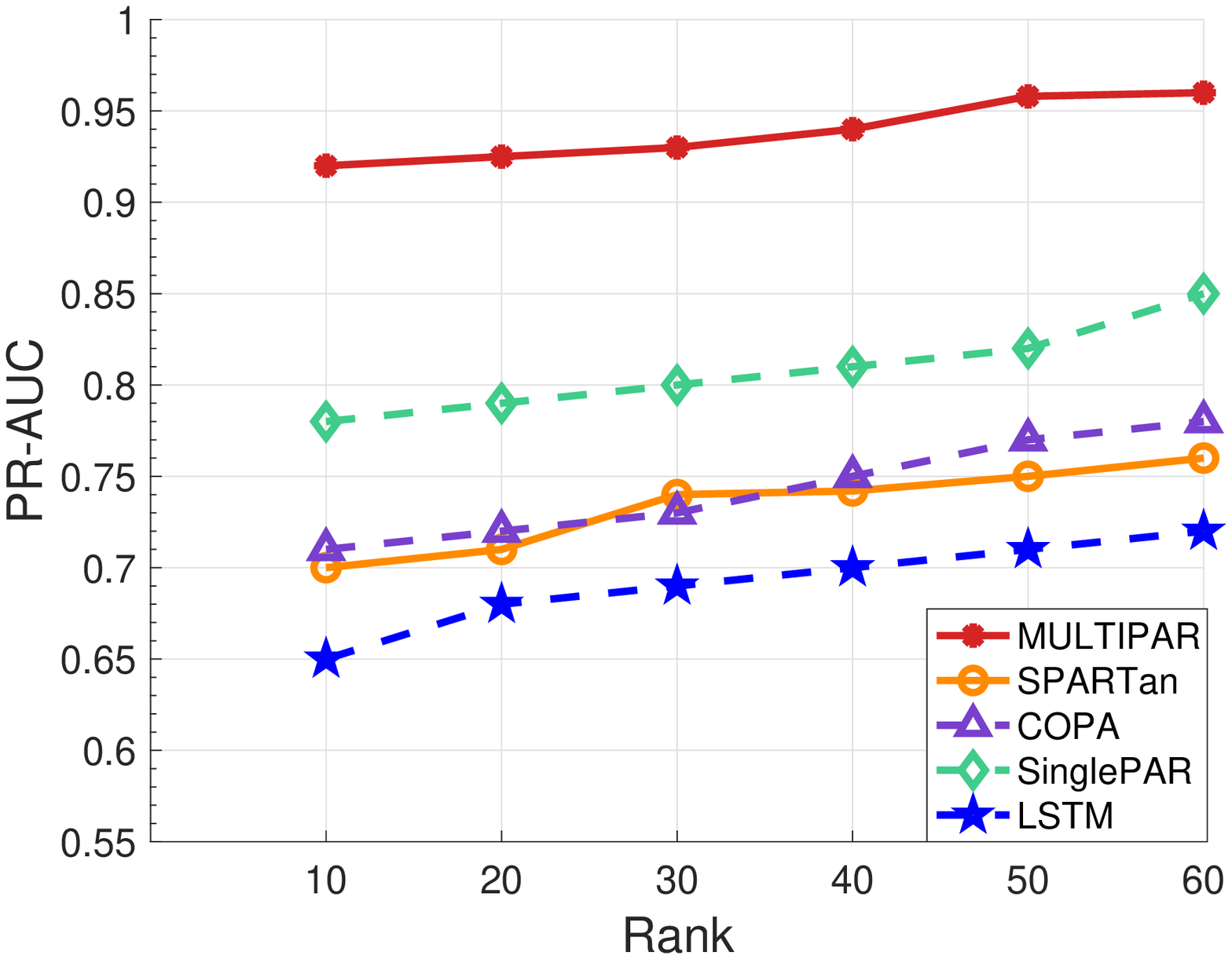}
        \caption{Dialysis Prediction Task}
        \label{fig:eicudialysisauc}
    \end{subfigure}
        \begin{subfigure}{0.3\textwidth}
        \includegraphics[width=\textwidth]{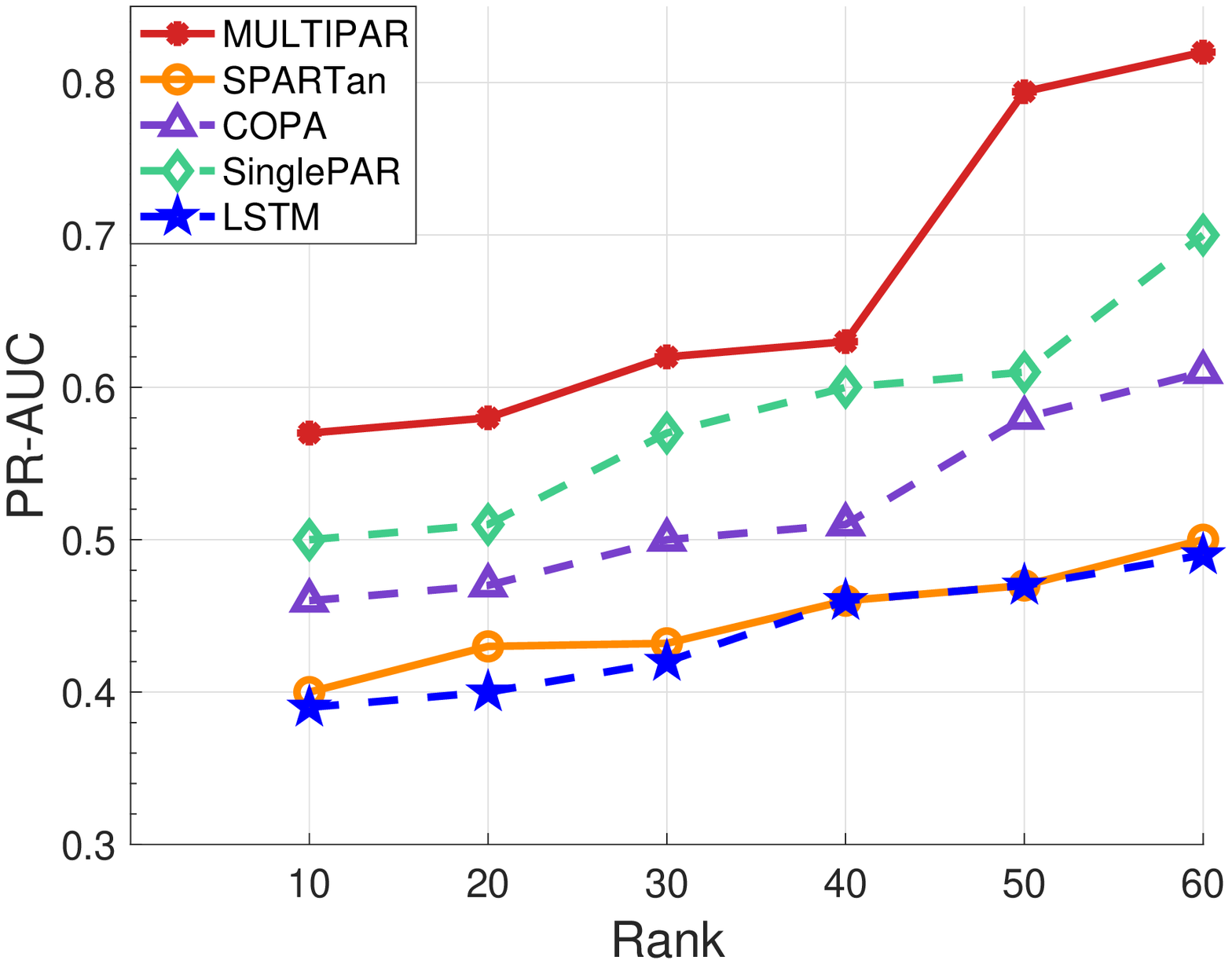}
        \caption{Venti-res Prediction Task}
        \label{fig:eicuventiauc}
    \end{subfigure}
        \begin{subfigure}{0.3\textwidth}
        \includegraphics[width=\textwidth]{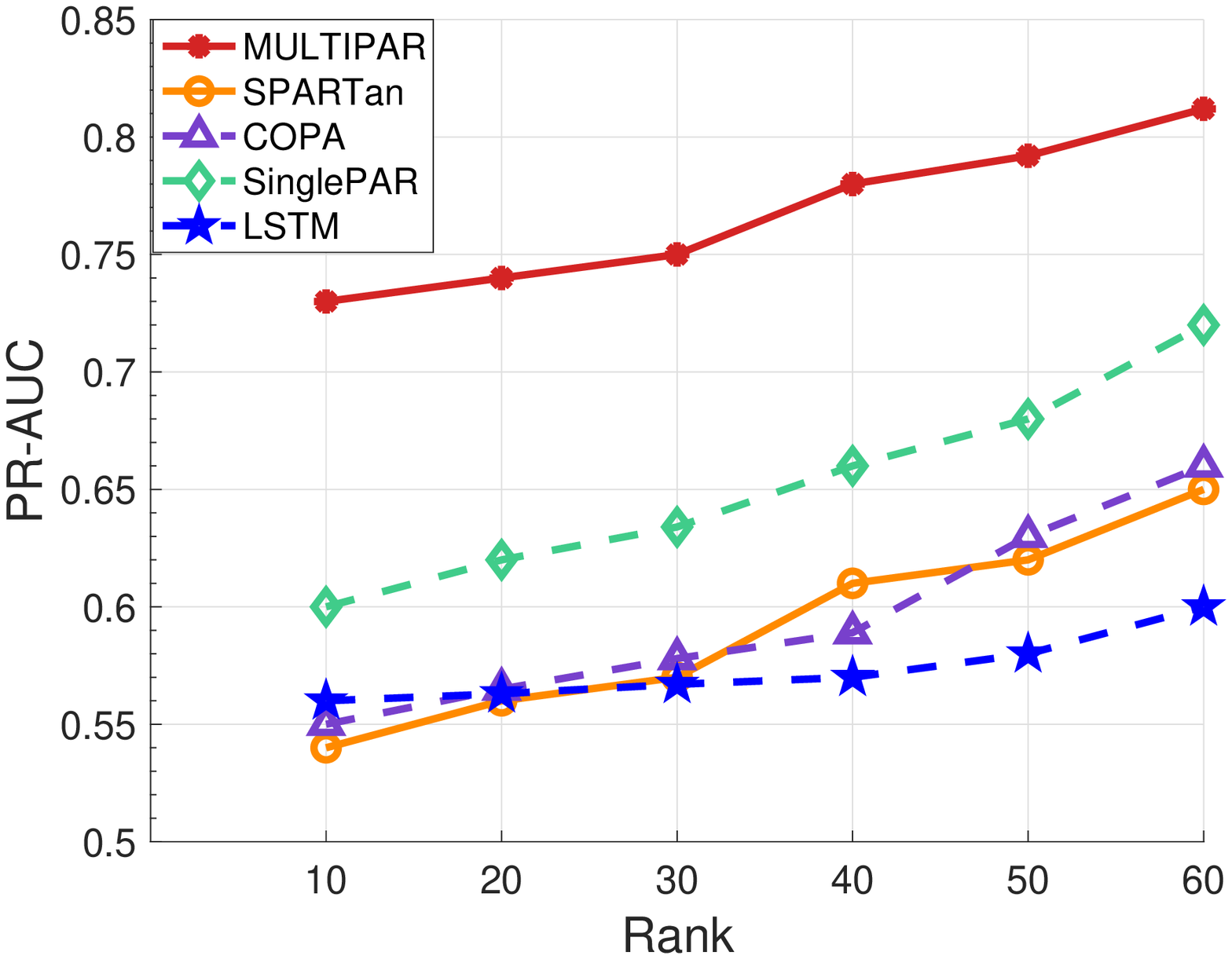}
        \caption{Intubated Prediction Task}
        \label{fig:eicuintubbatedauc}
    \end{subfigure}

    \caption{PR-AUC for prediction tasks on eICU dataset}\label{fig:auceICU}
\end{figure}




\partitle{Predictability analysis}
A logistic regression model is trained on the patient importance membership matrix $\S_k$ for static outcome prediction tasks and an LSTM model is trained on the temporal evolution matrix $\U_k$ for dynamic outcome prediction task. LSTM is a variant of the recurrent neural network (RNN) that mitigates the gradient vanishing problem in traditional RNNs.

In the MIMIC-EXTRACT dataset, only the ventilation prediction task is a dynamic task, and all the tasks in the eICU dataset are static tasks. In order to illustrate the benefit of using the latent factors as features for a downstream prediction model, we also include an LSTM model trained using the original EHR data. The reason why we choose LSTM model is because the original EHR data contains different length patients' visit data, and LSTM model can handle the varying size input temporal data. The input to the LSTM model is an irregular tensor which contains $k$ different patients, and each patient information $X_k$ consists of $I_K$ visits and $J$ medical features. The output is the prediction label for the different patients and different visit for the dynamic task.

We evaluate the prediction accuracy as a function of the tensor factorization rank. As shown in Figures \ref{fig:aucEXTRACT} and \ref{fig:auceICU}, MULTIPAR outperforms the other methods. In Figure \ref{fig:aucEXTRACT}, when the rank is 10, MULTIPAR outperforms the best baseline methods SinglePAR by 17\%, 18\%, 20\% and 22\% for each of the tasks respectively. This demonstrates MULTIPAR's strong generalization ability across multiple prediction tasks by leveraging the shared information between different tasks as well as the strong predictive power of the extracted phenotypes. 
Moreover, SinglePAR always outperforms COPA, SPARTan, and LSTM, which shows that the supervised learning framework can improve predictability. The Figure also illustrates the important role PARAFAC2 plays as the non-tensor based LSTM model performs the worst because it lacks the ability to filter out noise in the raw EHR.

\begin{figure}
    \centering
    \begin{subfigure}{0.3\textwidth}
        \includegraphics[width=\textwidth]{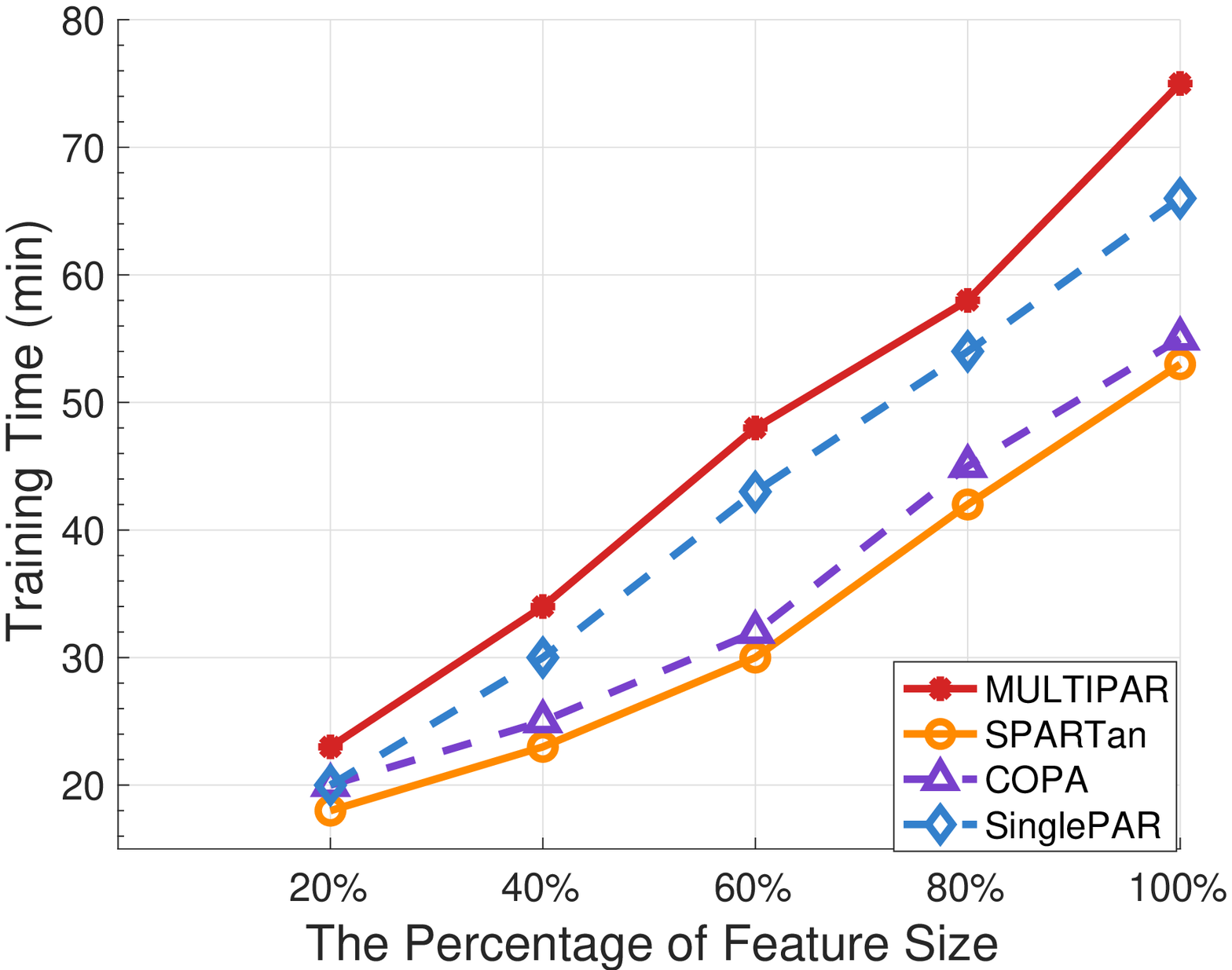}
        \caption{ MIMIC-EXTRACT varying patient size}
        \label{fig:scalaextractpatient}
    \end{subfigure}
     \begin{subfigure}{0.3\textwidth}
        \includegraphics[width=\textwidth]{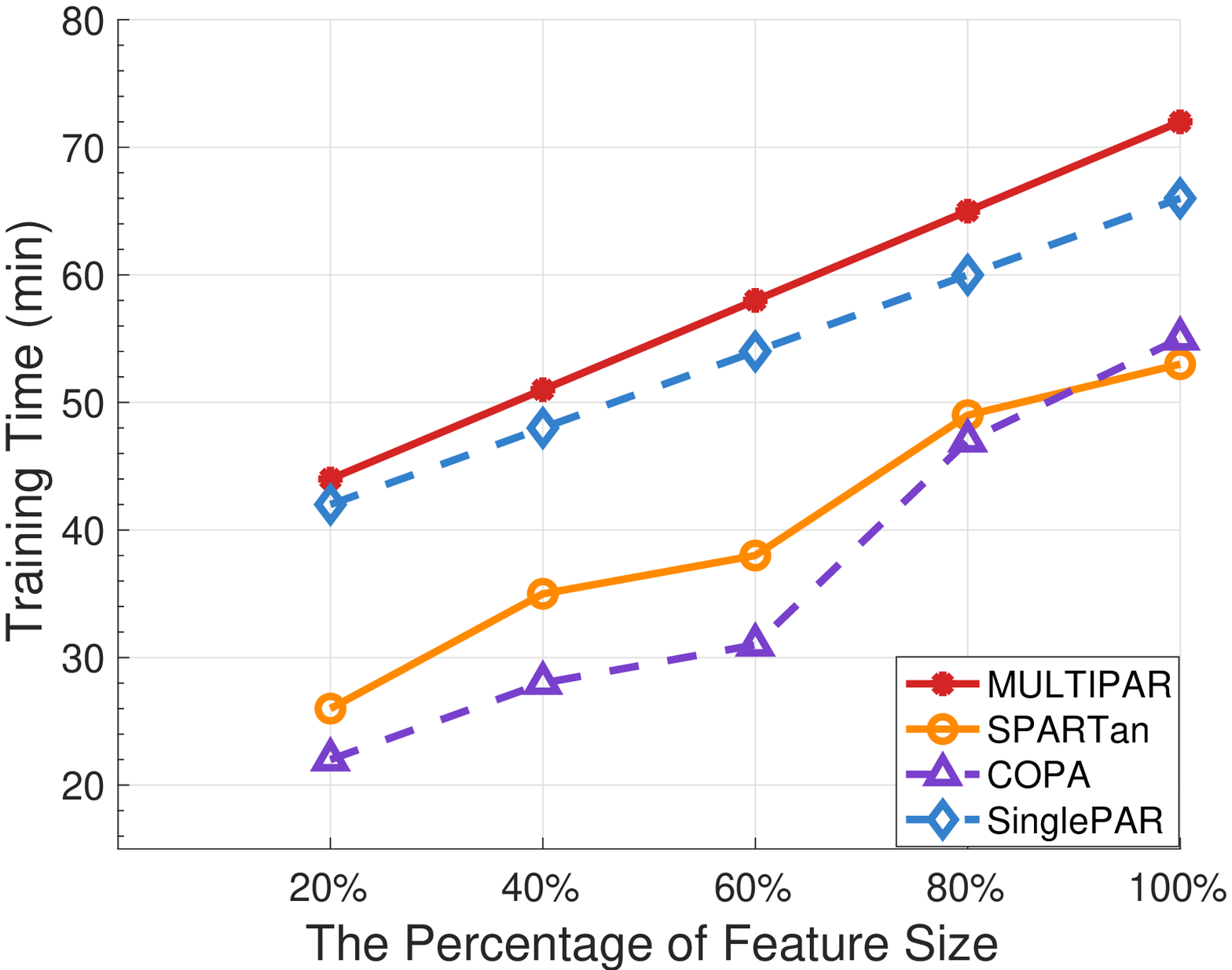}
        \caption{MIMIC-EXTRACT varying feature size}
        \label{fig:scalaextractfeature}
    \end{subfigure}\\
        \begin{subfigure}{0.3\textwidth}
        \includegraphics[width=\textwidth]{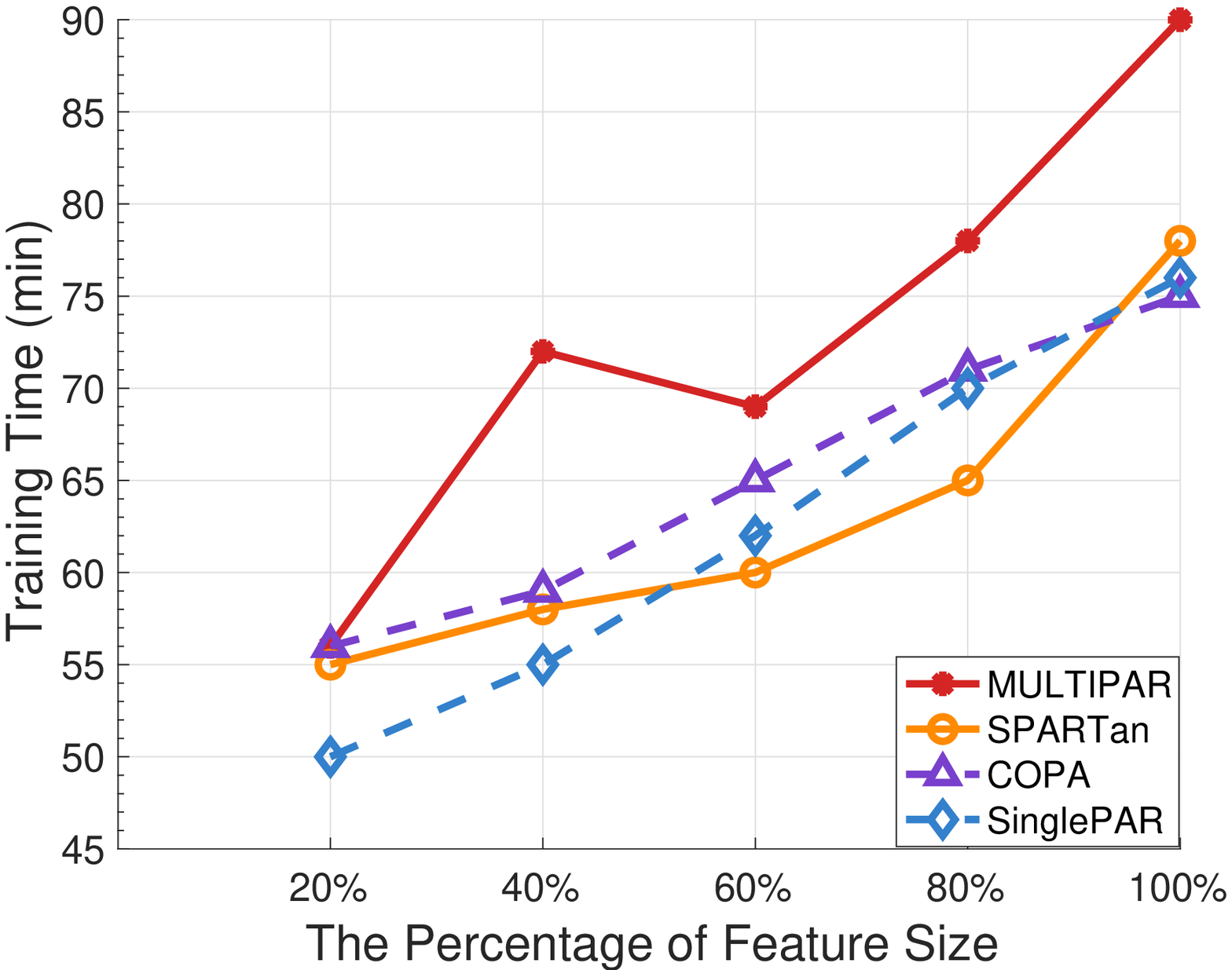}
        \caption{eICU varying patient size}
        \label{fig:scalaeICUpatient}
    \end{subfigure}
        \begin{subfigure}{0.3\textwidth}
        \includegraphics[width=\textwidth]{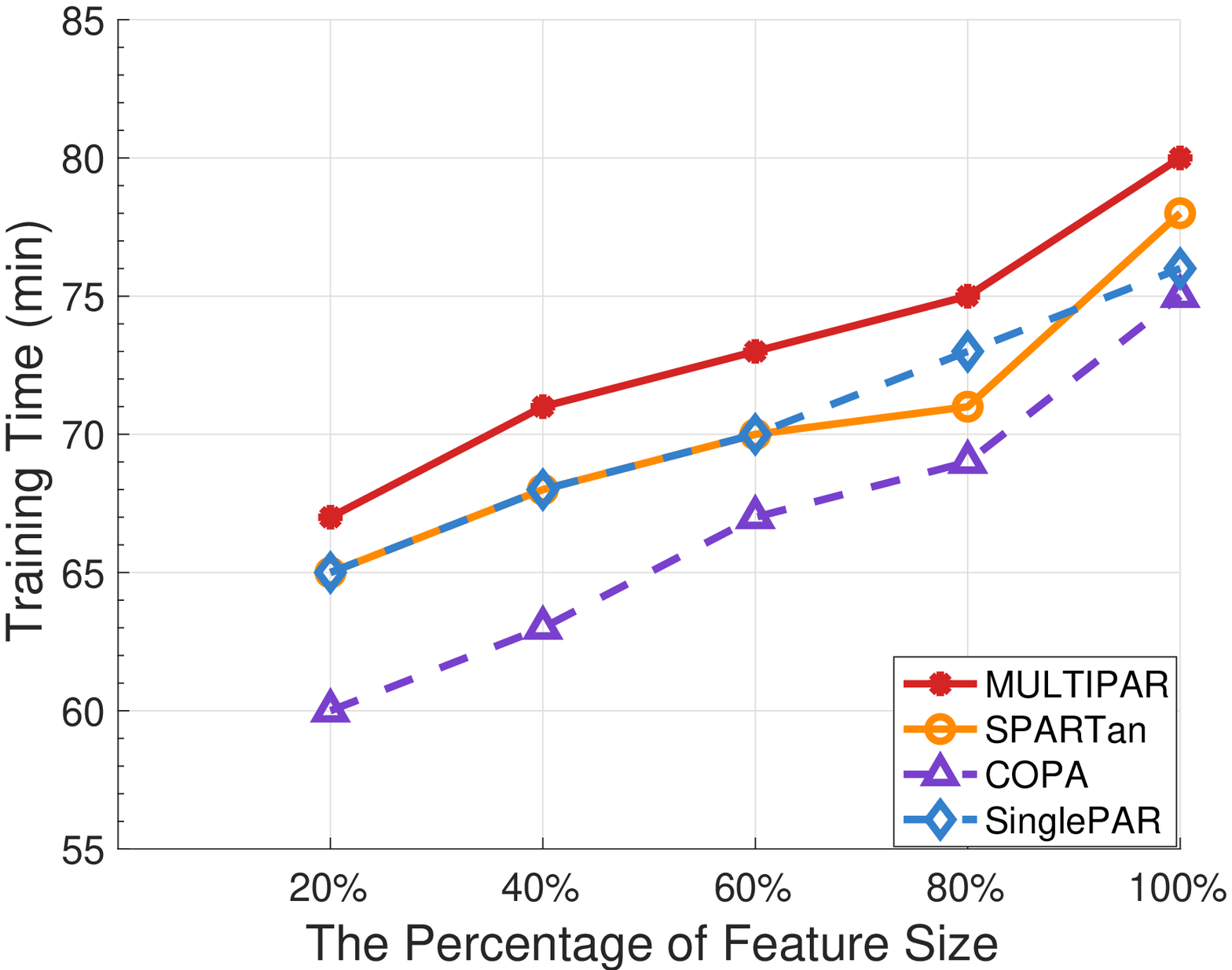}
        \caption{eICU varying feature size}
        \label{fig:scalaeICUfeature}
    \end{subfigure}
        \caption{Training time on MIMIC-EXTRACT and eICU varying patient size and feature size }\label{fig:scala}
\end{figure}

\partitle{Scalability analysis}
 Adding MTL on the PARAFAC2 framework can raise potential scalability issues on large datasets. Therefore, we evaluated the computational time of MULTIPAR compared with the other baseline methods using different data sizes and different feature sizes. We use two Titan RTX GPUs, each GPU has 24 GB of RAM, and train 50 epochs of both methods.
 
 Figure \ref{fig:scala} shows the total training time. COPA, SPARTan, and SinglePAR shows linear scalability as the number of patients and features grows. Although MULTIPAR adds MTL, it still exhibits linear scalability similar to SinglePAR. While MTL adds some additional training time, it is not significantly more as the maximum added time is 8 minutes. Moreover, the experiment result is consistent with our computational complexity analysis, the per-iteration computational complexity is linear with respect to number of patients  $J$ and feature size $K$.
 
\begin{table}
\centering
\resizebox{0.5\textwidth}{!}{%
\begin{tabular}{l c c} \toprule 
\textbf{Phenotype 1 (Normal  vital signs)}& \textbf{Weight}&\textbf{Average Value}\\ \midrule
Oxygen saturation & 1.52 &  98.5\\ 
Systolic blood pressure & 0.91 & 112.7 \\ 
Heart rate & 0.82 &  82.5 \\ 
Mean blood pressure & 0.79 &  81.2 \\ 
Diastolic blood pressure & 0.65 & 76.3  \\ 
Respiratory rate & 0.57 &  18.6  \\ 
Co2 (etco2, pco2, etc.) & 0.43 &  24.2 \\ \hline
\textbf{Phenotype 2 (Abnormal renal and liver function)}& \textbf{Weight}&\textbf{Average Value}\\ \midrule
Alanine aminotransferase & 11.51 &  83.1\\ 
Blood urea nitrogen & 9.64 & 42.3 \\ 
Alkaline phosphate & 8.01 &  153.2 \\ 
Asparate aminotransferase & 5.18 &  90.1 \\ 
Albumin & 3.90 & 3.2 \\ 
Bicarbonate & 2.76 &  17  \\ 
Mean blood pressure & 1.59 &  85  \\\hline
\textbf{Phenotype 3 (Normal  Blood Counts and Serum Electrolytes)}& \textbf{Weight}&\textbf{Average Value}\\ \midrule
Mean corpuscular hemoglobin concentration & 7.54 &  32.1 \\ 
Sodium & 4.93 & 135.2 \\ 
Mean corpuscular hemoglobin & 3.62 &  30.8  \\ 
Mean corpuscular volume & 3.41 &  93.2 \\ 
Chloride & 2.73 & 103 \\ 
Hemoglobin & 1.04 &  12.8  \\ 
Hematocrit & 0.62 &  33.2 \\ \hline
\textbf{Phenotype 4 (Abnormal vital signs)}& \textbf{Weight}&\textbf{Average Value}\\ \midrule
Glascow coma scale total & 2.13 &  6.7\\ 
Oxygen saturation & 1.41 & 85 \\ 
Systolic blood pressure & 1.30 &  153.1\\ 
Temperature & 1.29 &  37.5  \\ 
Heart rate & 1.03 & 115 \\ 
Mean blood pressure & 0.93 &  95  \\ 
Diastolic blood pressure & 0.84 &  82 \\
 \bottomrule
\end{tabular}}
\caption{Phenotypes discovered by MULTIPAR.}\label{tbl:phenoextract}
\end{table}

\begin{table}
\caption{MIMIC-EXTRACT phenotypes discovered by SinglePAR incorporating in-hospital mortality prediction.}
\centering
\resizebox{0.28\textwidth}{!}{%
\begin{tabular}{l l }
\hline
\textbf{Phenotype 1} \\ \hline
{Oxygen saturation } & {Systolic blood pressure}\\
{Heart rate} & {PH }\\
{Mean blood pressure } & {Diastolic blood pressure}\\\hline
\textbf{Phenotype 2} \\ \hline
{Oxygen saturation} & {Systolic blood pressure }\\
{PH} & {Mean blood pressure}\\
{Heart rate } & {Diastolic blood pressure } \\\hline
\textbf{Phenotype 3} \\ \hline
{Temperature  } & {Glascow coma scale total}\\
{Oxygen saturation } & {Systolic blood pressure } \\
{Heart rate} & {Mean blood pressure }\\\hline
\textbf{Phenotype 4} \\ \hline
{Glascow coma scale total } & {Heart rate}\\
{Temperature} & {Systolic blood pressure}\\
{Mean blood pressure} & {PH}\\\hline
\end{tabular}}
\label{tbl:phenoSinglePARinhos}
\end{table}

\begin{table}
\caption{MIMIC-EXTRACT phenotypes discovered by incorporating icu mortality prediction.}
\centering
\resizebox{0.28\textwidth}{!}{%
\begin{tabular}{l l }
\hline
\textbf{Phenotype 1} \\ \hline
{Oxygen saturation } & {Systolic blood pressure}\\
{Heart rate} & {respiratory rate }\\
{Mean blood pressure } & {Diastolic blood pressure}\\\hline
\textbf{Phenotype 2} \\ \hline
{Hemoglobin} & {PH }\\
{Sodium} & {chloride}\\
{Mean corpuscular volume} & {Co2 (etco2, pco2, etc.) } \\\hline
\textbf{Phenotype 3} \\ \hline
{Oxygen saturation} & {Systolic blood pressure}\\
{Heart rate } & {Respiratory rate} \\
{Mean blood pressure} & {Diastolic blood pressure }\\\hline
\textbf{Phenotype 4} \\ \hline
{Temperature } & {Glascow coma scale total}\\
{Oxygen saturation} & {Systolic blood pressure}\\
{Heart rate} & {Mean blood pressure}\\\hline
\end{tabular}}
\label{tbl:phenoicu}
\end{table}

\begin{table}
\caption{MIMIC-EXTRACT phenotypes discovered by SinglePAR incorporating readmission prediction.}
\centering
\resizebox{0.28\textwidth}{!}{%
\begin{tabular}{l l }
\hline
\textbf{Phenotype 1} \\ \hline
{Temperature } & {Glascow coma scale total}\\
{Oxygen saturation} & {Systolic blood pressure}\\
{Heart rate} & {Mean blood pressure}\\\hline
\textbf{Phenotype 2} \\ \hline
{Oxygen saturation} & {Systolic blood pressure }\\
{Heart rate} & {Mean blood pressure}\\
{Diastolic blood pressure } & {Respiratory rate } \\\hline
\textbf{Phenotype 3} \\ \hline
{Sodium } & {Chloride}\\
{Oxygen saturation } & {PH } \\
{Hemoglobin} & {Hear rate }\\\hline
\textbf{Phenotype 4} \\ \hline
{Oxygen saturation } & {Systolic blood pressure}\\
{Heart rate} & {Mean blood pressure}\\
{Diastolic blood pressure} & {PH}\\\hline
\end{tabular}}
\label{tbl:phenoSinglePARread}
\end{table}

\begin{table}
\caption{MIMIC-EXTRACT phenotypes discovered by SinglePAR incorporating ventilation prediction}
\centering
\resizebox{0.28\textwidth}{!}{%
\begin{tabular}{l l }
\hline
\textbf{Phenotype 1} \\ \hline
{Oxygen saturation } & {Systolic blood pressure}\\
{Heart rate} & {Respiratory rate }\\
{Mean blood pressure } & {Diastolic blood pressure}\\\hline
\textbf{Phenotype 2} \\ \hline
{Respiratory rate} & {Sodium }\\
{Temperature} & {Mean corpuscular volume}\\
{Chloride } & {PH } \\\hline
\textbf{Phenotype 3} \\ \hline
{Oxygen saturation } & {Systolic blood pressure}\\
{Heart rate } & {Mean blood pressure } \\
{Diastolic blood pressure} & {Respiratory rate }\\\hline
\textbf{Phenotype 4} \\ \hline
{Temperature } & {Glascow coma scale total}\\
{Oxygen saturation} & {Diastolic blood pressure}\\
{Heart rate} & {Mean blood pressure}\\\hline
\end{tabular}}
\label{tbl:phenoSinglePARventi}
\end{table}

\begin{table}
\caption{MIMIC-EXTRACT phenotypes discovered by COPA.}
\centering
\resizebox{0.35\textwidth}{!}{%
\begin{tabular}{l l }
\hline
\textbf{Phenotype 1} \\ \hline
{Mean corpuscular hemoglobin concentration } & {Sodium}\\
{Mean corpuscular hemoglobin} & {Oxygen saturation } \\
{Mean corpuscular volume } & {Chloride }\\\hline
\textbf{Phenotype 2} \\ \hline
{Temperature} & {Oxygen saturation } \\
{Systolic blood pressure} & {Heart rate }\\
{Mean blood pressure } & {Diastolic blood pressure}\\\hline
\textbf{Phenotype 3} \\ \hline
{Glascow coma scale total } & {Temperature}\\
{Oxygen saturation } & {Systolic blood pressure }\\
{Heart rate} & {Mean blood pressure}\\\hline
\textbf{Phenotype 4} \\ \hline
{Oxygen saturation} & {Systolic blood pressure}\\
{Mean blood pressure} & {Diastolic blood pressure }\\
{Heart rate} & {Respiratory rate}\\ \hline
\end{tabular}}
\label{tbl:phenocopa}
\end{table}

\begin{figure}
    \centering
    \begin{subfigure}{0.3\textwidth}
        \includegraphics[width=\textwidth]{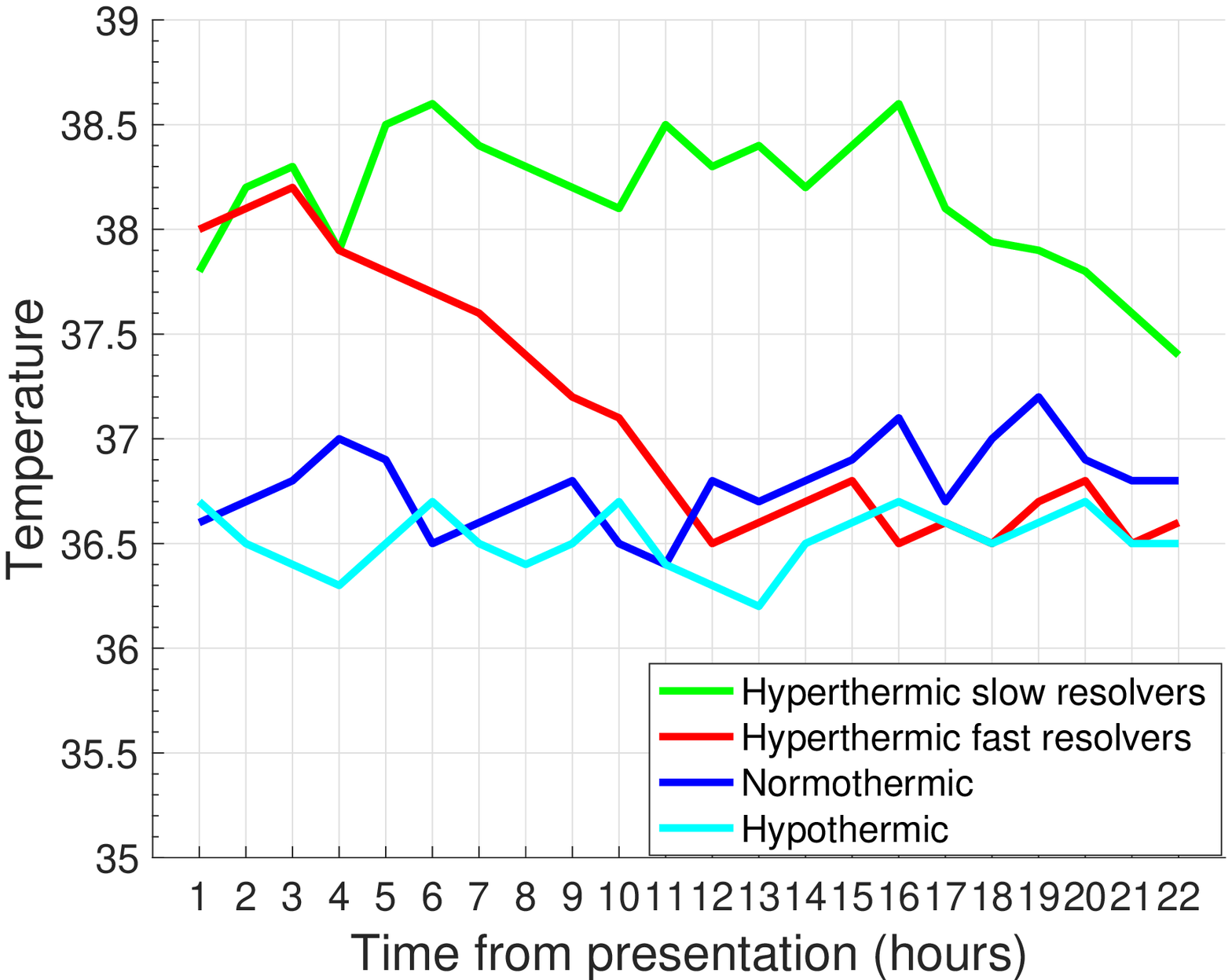}
        \caption{Temperature}
        \label{fig:temperature}
    \end{subfigure}
     \begin{subfigure}{0.3\textwidth}
        \includegraphics[width=\textwidth]{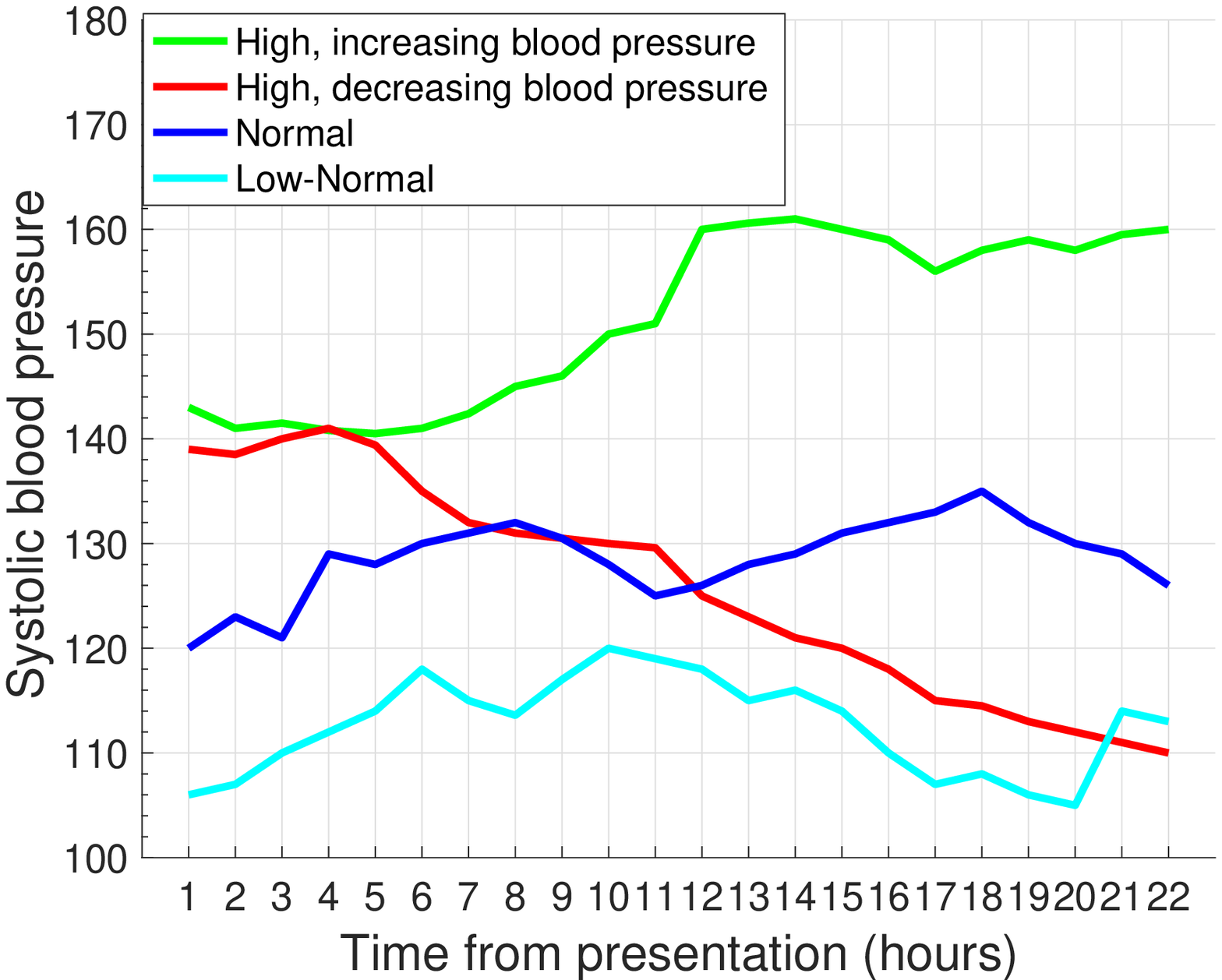}
        \caption{Blood pressure}
        \label{fig:systolic blood pressure}
    \end{subfigure}
        \caption{Temporal Trajectory}\label{fig:temporal}
\end{figure}

\partitle{Interpretability analysis}
Finally, we did an interpretability analysis of MULTIPAR on the MIMIC-EXTRACT dataset. We first illustrate the phenotypes discovered by MULTIPAR in Table \ref{tbl:phenoextract}. We set rank to $4$, and use the $\V$ matrix to select the most important vital signs in each phenotype based on the weight. $\V$ matrix represents the membership of medical features in each phenotype, and the “weight” column in Table \ref{tbl:phenoextract} is the weight in the $V$ matrix. We then use the $\S$ matrix to find the patient subgroup of each phenotype, and calculate the average value of the vital signs shown in the “Average value” column. It is important to note that there is no post-processing in these extracted phenotypes. A critical care expert reviewed and endorsed the presented phenotypes which suggest collective characteristics such as normal vital signs, abnormal renal and liver function, normal blood counts and serum electroytes, and abnormal vital signs. 


The phenotypes discovered by the supervised single task model SinglePAR strongly overlap with each other as shown in Tables \ref{tbl:phenoSinglePARinhos}, \ref{tbl:phenoicu}, \ref{tbl:phenoSinglePARread}, and \ref{tbl:phenoSinglePARventi}. Since we are incorporating in-hospital mortality prediction task in Table \ref{tbl:phenoSinglePARinhos}, most of the phenotypes discovered by SinglePAR are abnormal in vital signs. COPA discovered phenotypes shown in Table \ref{tbl:phenocopa} contain more information compared to SinglePAR, which makes sense because a supervised model may guide the tensor factorization to a specific direction geared toward the task and cause information loss. However, MULTIPAR does not have information loss compared to COPA, it even provides a new phenotype (phenotype 2: abnormal in renal and liver function) which is not discovered by COPA. This verifies the benefit of MTL in MULTIPAR, which can leverage information from multiple tasks to avoid local optimum. 

We then test MULTIPAR's ability to find meaningful subgroup temporal trajectories, which can help clinical care experts make precise prescriptions and treatments for specific subgroup of patients. We select the patients with the number of observations equal to $22$ for visualization purposes. We select the four phenotypes for the temperature feature and systolic blood pressure feature, then use the $\S$ matrix to find the patient subgroup for each phenotype, and print the average value trajectory.  

From Figure \ref{fig:temporal}, we can see that the four patient subgroups (clusters) exhibit very different temporal trajectories in the temperature. Our clinical expert interpreted that the green line suggests a hyperthermic slow resolver patient subgroup  which exhibits a slow decreasing trend as time increases, the red line suggests a hyperthermic fast resolver patient subgroup, which exhibits a fast decreasing trend as time increases, the dark blue line suggests a normothermic patient subgroup  and the light blue line is a hypothermic patient subgroup. For the systolic blood pressure trajectory shown in Figure \ref{fig:systolic blood pressure}, the green subgroup has high, increasing blood pressure, the red subgroup has high, decreasing blood pressure, the dark blue and light blue subgroups have consistently normal  and low-normal blood pressure, respectively. 

\section{Conclusion}
PARAFAC2 irregular tensor factorization has been studied to successfully extract meaningful medical concepts (phenotypes). Despite recent advancements, the predictability and interpretability is not satisfactory. We proposed MULTIPAR: a supervised irregular tensor factorization with multi-task learning to jointly optimize the tensor factorization and multiple downstream prediction tasks. It is built on three major contributions: a supervised framework for PARAFAC2 tensor factorization and downstream prediction tasks; a new multi-task learning  method combining tensor factorization and multiple prediction models; and a novel weight selection method for supervised multi-task optimization. We conducted extensive experiments to demonstrate that MULTIPAR can extract more meaningful phenotypes with higher predictability compared to state-of-the-art methods in a scalable way. In the future, we plan to incorporate low-rankness constraints for robustness against missing data, and incorporate more sophisticated regularization constraints to capture the complex and non-linear temporal relationships.

\bibliographystyle{unsrt}  
\bibliography{MULTIPAR}  

\end{document}